\documentclass[11pt]{article}

\usepackage{geometry} % to change the page dimensions
\usepackage[utf8]{inputenc} % set input encoding (not needed with XeLaTeX)
\usepackage[english]{babel}

\usepackage{standalone}
\usepackage{amsmath,amssymb,amsthm} % parceque je fais des maths
\usepackage{dsfont}

\usepackage{graphicx} % support the \includegraphics command and options

\usepackage{multicol}
\usepackage{booktabs} % for much better looking tables
\usepackage{array} % for better arrays (eg matrices) in maths
\usepackage{paralist} % very flexible & customisable lists (eg. enumerate/itemize, etc.)
\usepackage{algorithm}
\usepackage{algorithmic}
\usepackage{url}
\usepackage{hyperref}

\usepackage{arydshln}
\usepackage{pdflscape}

\usepackage{multirow}
\usepackage{tracefnt}
\newcommand\coolrightbrace[2]{%
\left.\vphantom{\begin{matrix}%
#1 \end{matrix}}\right\}#2}

\newcommand{\mx}[1]{\boldsymbol{\mathbf{#1}}}
\newcommand{\operator}[1]{\mathcal{#1}}
\newcommand{\cone}[1]{\text{cone}(#1)}
\newcommand{\coneT}[1]{\emph{\text{cone}}(#1)}
\newcommand{\R}{\mathbb{R}}
\newcommand{\N}{\mathbb{N}}

\theoremstyle{plain}
\newtheorem{theo}{Theorem}
\newtheorem{lemma}[theo]{Lemma}
\newtheorem{proposition}[theo]{Proposition}

\theoremstyle{definition}
\newtheorem{definition}{Definition}
\newtheorem*{rem}{Remark}

\title{Nonnegative matrix factorization with side information for time series recovery and prediction}

\author{Jiali~Mei
        \and Yohann~De~Castro
        \and Yannig~Goude
        \and Jean-Marc~Azaïs
        \and Georges~Hébrail
\thanks{ J. Mei, Y. Goude and G. Hébrail are with EDF Lab Paris-Saclay, 91120 Palaiseau, France.
 J. Mei, Y. De Castro, and Y. Goude are with Laboratoire des Mathématiques d'Orsay, Univ.
Paris-Sud, CNRS, Université Paris-Saclay, 91405 Orsay, France.
 J-M. Azaïs is with Institut de Mathématiques, Université Paul Sabatier 31062 Toulouse, France}% <-this % stops an unwanted space
}

\begin{document}
\maketitle

\begin{abstract}
Motivated by the reconstruction and the prediction of electricity consumption, we extend Nonnegative Matrix Factorization~(NMF) to take into account side information (column or row features).
We consider general linear measurement settings, and propose a framework which models non-linear relationships between features and the response variables.
We extend previous theoretical results to obtain a sufficient condition on the identifiability of the NMF in this setting.
Based the classical Hierarchical Alternating Least Squares~(HALS) algorithm, we propose a new algorithm (HALSX, or Hierarchical Alternating Least Squares with eXogeneous variables) which estimates the factorization model.
The algorithm is validated on both simulated and real electricity consumption datasets as well as a recommendation dataset, to show its performance in matrix recovery and prediction for new rows and columns.
\end{abstract}

\section{Introduction}\label{sec:introduction}

% \abstract{
% Motivated by the reconstruction and prediction of electricity consumption, we extend Nonnegative Matrix Factorization (NMF) to take into account outside features.
% We consider Nonnegative Matrix Factorization in general linear measurement schemes, and propose a general framework which models non-linear relationship between features and the response variables.
% We extend previous theoretical results in NMF to obtain a sufficient condition on the identifability of matrix factorization.
% Based the classical Hierarchical Alternating Least Squares (HALS) algorithm, we propose a new algorithm (HALSX, or Hierarchical Alternating Least Squares with eXogeneous variables) which estimates the factorization model.
% The algorithm is validated on both simulated and real electricity consumption data, to show its performance in reconstruction and prediction.}
% This technical report summarizes the work on nonnegative matrix factorization with side information. It has two major components. 
% First, we present a sufficient condition for identifiability of nonnegative matrix factorization with side information. 
% Second, we present a novel algorithm to integrate side information into matrix factorization using spline bases, and present an empirical study on electricity consumption estimation and forecasting using this algorithm.}

% \section{Introduction}

In recent years, a large number of methods have been developed to improve matrix completion methods using side information \cite{jain_provable_2013,rao_collaborative_2015,si_goal-directed_2016}.
By including features linked to the row and/or columns of a matrix, also called ``side information'', these methods have a better performance at estimating the missing entries.

In this work, we generalize this idea to nonnegative matrix factorization (NMF, \cite{donoho_when_2003}).
Although more difficult to identify, NMF often has a better empirical performance, and provides factors that are more interpretable in an appropriate context.
We propose an NMF method that takes into account side information.
Given some observations of a matrix, this method jointly estimates the nonnegative factors of the matrix, and regression models of these factors on the side information.
This allows us to improve the matrix recovery performance of NMF.
Moreover, using the regression models, we can predict the value of interest for new rows and columns that are previously unseen.
We develop this method in the general matrix recovery context, where linear measurements are observed instead of matrix entries.

This choice is especially motivated by applications  in electricity consumption.
We are interested in estimating and predicting the electricity load from temporal aggregates. 
In the context of load balancing in the power market, electric transmission system operators (TSO) of the electricity network are typically legally bound to estimate the electricity consumption and production at a small temporal scale (half-hourly or hourly), for market participants within their perimeter, \textit{i.e.} utility providers, traders, large consumers, groups of consumers, \textit{etc.} \cite{mei_nonnegative_2017}. 
Most traditional electricity meters do not provide information at such a fine temporal scale. 
Although smart meters can record consumption locally every minute, the usage of such data can be extremely constrained for TSOs, because of the high cost of data transmission and processing and/or privacy issues.
Nowadays, TSOs often use regulator-approved proportional consumption profiles to estimate market participants' load.
In a previous work by the authors
\cite{mei_nonnegative_2017}, we proposed to solve the estimation problem by NMF using temporal aggregate data.

Using the method developed in this article, we put in parallel temporal aggregate data with features that are known to have a correlation with electricity consumption, such as the temperature, the day of the week, or the type of client.
This not only improves the performance of load estimation, but also allows us to predict future load for users in the dataset, and estimate and predict the consumption of new users previously unseen.
In electrical power networks, load prediction for new periods is useful for balancing offer-demand on the network, and prediction for new individuals is useful for network planning.

In the rest of this section, we introduce the general framework of this method, and the related literature.
In Section \ref{sec:identifiability}, we deduce a sufficient condition on the side information for the NMF to be unique.
In Section \ref{sec:algorithm}, we present HALSX, an algorithm which solves the NMF with side information problem, that we prove to converge to stationary points.
In Section \ref{sec:experiments}, we present experimental results, applying HALSX on simulated and real datasets, both for the electricity consumption application, and for a standard task in collaborative filtering.

\subsection{General model definition\label{sec:general model}}

We are interested in reconstructing a nonnegative matrix~$\mx{V}^\ast \in \R_+^{n_1 \times n_2}$, from~$N$ linear measurements,
\begin{align}
\mx{\alpha} = \operator{A}(\mx{V}^\ast) \in \R^{N},
\end{align}
where~$\operator{A} : \R^{n_1 \times n_2} \rightarrow \R^{N}$ is a linear operator.
Formally,~$\operator{A}$ can be represented by~$\mx{A}_1$, ...,~$\mx{A}_N$,~$N$ design matrices of dimension~$n_1 \times n_2$, and each linear measurement can be represented by
\begin{align}
\alpha_i = \text{Tr}(\mx{V}^\ast \mx{A}_i^T) = \langle \mx{V}^\ast, \mx{A}_i \rangle.
\end{align}
The design matrices~$\mx{A}_1$, ...,~$\mx{A}_N$ are called \textit{masks}.

% \subsubsection*{Generative low-rank nonnegative model}
Moreover, we suppose that the matrix of interest,~$\mx{V}^\ast$, stems from a generative low-rank nonnegative model, in the following sense:
\begin{enumerate} 
  \item The matrix~$\mx{V}^\ast$ is of \textit{nonnegative rank~$k$}, with~$k \ll n_1, n_2$. 
  This means, $k$ is the smallest number so that we can find two nonnegative matrices~$\mx{F}_r \in \R_+^{n_1 \times k}~$ and~$\mx{F}_c \in \R_+^{n_2 \times k}$ satisfying 
      \begin{align*}
        \mx{V}^\ast = \mx{F}_r \mx{F}_c^T.
      \end{align*}
  Note that this implies that~$\mx{V}^\ast$ is of rank at most~$k$, and therefore is of low rank.
  \item There are some row features~$\mx{X}_r \in \R^{n_1 \times d_1}$ and column features~$\mx{X}_c \in \R^{n_2 \times d_2}$ connected to each row and column of~$\mx{V}^\ast$. 
      We note by~$\mx{x}_r^i$ the~$i$-th row of~$\mx{X}_r$, and by~$\mx{x}_c^i$ the~$i$-th row of~$\mx{X}_c$.
      There are two link functions~$f_r: \R^{d_1} \rightarrow \R^k$ and~$f_c: \R^{d_1} \rightarrow \R^k$, so that
      \begin{align*}
        &\mx{F}_r = (f_r(\mx{X}_r))_+, \\
        &\mx{F}_c = (f_c(\mx{X}_c))_+,
      \end{align*}
      where~$f_r(\mx{X}_r) \in \R^{n_1 \times k}$ is the matrix obtained by stacking row vectors~$f_r(\mx{x}_r^i)$, for~$1 \leq i \leq n_1$ (\textit{idem} for~$f_c(\mx{X}_c) \in \R^{n_n \times k}$), and~$(\cdot)_+$ is the ramp function which corresponds to thresholding operation at 0 for any matrix or vector.
\end{enumerate}

In this general setting, the features~$\mx{X}_r$ and~$\mx{X}_c$, the measurement operator~$\operator{A}$, and the measurements~$\mx{\alpha}$ are observed.
The objective is to estimate the true matrix~$\mx{V}^\ast$ as well as the factor matrices~$\mx{F}_r$ and~$\mx{F}_c$, by estimating the link functions~$f_r$ and~$f_c$.

To obtain a solution to this matrix recovery problem, we minimize the quadratic error of the matrix factorization. 
In Section \ref{sec:algorithm}, we will propose an algorithm for the following optimization problem:
\begin{equation}\label{eq:general_optimization_problem}
\begin{aligned}
\min_{\mx{V},f_r \in F_r^k, f_c \in F_c^k}\quad & \|\mx{V} - (f_r(\mx{X}_r))_+ (f_c(\mx{X}_c))_+^T\|_F^2\\
 \text{s.t.}\quad & \operator{A}(\mx{V}) = \mx{b}, \quad \mx{V} \geq \mx{0},
\end{aligned}
\end{equation}
where~$F_r \subseteq (\R)^{\R^{d_1}}$ and~$F_c \subseteq (\R)^{\R^{d_2}}~$ are some functional spaces in which the row and column link functions are to be searched.

By specializing~$\mx{X}_r$,~$\mx{X}_c$, and~$\operator{A}$, or restricting the search space of~$f_r$ and~$f_c$, this general model includes a number of interesting applications, old and new.

\subsubsection*{The masks~$\mx{A}_1$, ...,~$\mx{A}_N$}
\begin{itemize}
\item \textbf{Complete observation}:~$N = n_1 n_2, \mx{A}_{i_1, i_2} = \mx{e}_{i_1} \mx{e}_{i_2}^T$, where~$\mx{e}_{i}$ is the~$i$-th canonical vector. This means every entry of~$\mx{V}^\ast$ is observed.
\item \textbf{Matrix completion}: the set of masks is a subset of complete observation masks, with~$N < n_1 n_2~$.
\item \textbf{Matrix sensing}: the design matrices~$\mx{A}_i$ are random matrices, sampled from a certain probability distribution. 
Typically, the probability distribution needs to verify certain conditions, so that with a large probability,~$\operator{A}$ verifies the Restricted Isometry Property \cite{recht_guaranteed_2010}.
\item \textbf{Rank-one projections} \cite{cai_rop:_2015,zuk_low-rank_2015}: the design matrices are random rank-one matrices, that is $\mx{A}_i = \mx{\alpha}_i \mx{\beta}_i^T$, where $\mx{\alpha}_i$ and $\mx{\beta}_i$ are respectively random vectors of dimension $n_1$ and $n_2$.
The main advantage to this setting is that much less memory is needed to store the masks, since we can store the vectors $\mx{\alpha}_i$ and $\mx{\beta}_i$ (dimension-$(n_1 + n_2)$) instead of $\mx{A}_i$ (dimension-$(n_1\times n_2)$).
In \cite{cai_rop:_2015,zuk_low-rank_2015}, theoretical properties are proved for the case where $\mx{\alpha}_i$ and $\mx{\beta}_i$ are vectors with independent Gaussian entries and/or drawn uniformly from the vectors of the canonical basis.
\item \textbf{Temporal aggregate measurements}: in this case, the matrix is composed of~$n_1$ time series concerning~$n_2$ individuals, and each measure is a temporal aggregate of the time series of an individual. 
The design matrices are defined as
${\mx{A}_i = \sum_{t = t_0(i) + 1}^{t_0(i) + h(i)}\mx{e}_t \mx{e}_{s_i}^T}$,
where~$s_i$ is the individual concerned by the~$i$-th measure,~$t_0(i) + 1$ the first period covered by the measure, and~$h(i)$ the number of periods covered by the measure.
\end{itemize}

\subsubsection*{The features~$\mx{X}_r$ and~$\mx{X}_c$}
\begin{itemize}
\item \textbf{Individual features}:~$\mx{X}_r = \mx{I}_{n_1}, \mx{X}_c = \mx{I}_{n_2}$. Basically, no side information is available. The row individuals and column individuals are each different.
\item \textbf{General numeric features}:~$\mx{X}_r \in \R^{n_1 \times d_1}$ and~$\mx{X}_c \in \R^{n_2 \times d_2}$. This includes all numeric features.
\item \textbf{Features generated from a kernel}: certain information about the row and column individuals may not be in the form of a numeric vector. 
For example, if the row individuals are vertices of a graph, their connection to each other is interesting information for the problem, but it is difficult to encode as real vectors. 
In this case, features can be generated through a transformation, or by defining a kernel function.
\end{itemize}

\subsubsection*{The link functions~$f_r$ and~$f_c$}
\begin{itemize}
\item \textbf{Linear}:~$\mx{F}_r = f_r(\mx{X}_r) = \mx{X}_r \mx{B}_r$, and~$\mx{F}_c =f_c(\mx{X}_c) = \mx{X}_c \mx{B}_c$. 
In this case, we need to estimate~$\mx{B}_r$ and~$\mx{B}_c$ to fit the model.
With identity matrices as row and column features, this case is reduced to the traditional matrix factorization model with
\begin{align*}
\mx{F}_r = \mx{B}_r, \quad \mx{F}_c = \mx{B}_c, \quad
\mx{V}^\ast = \mx{F}_r \mx{F}_c^T = \mx{B}_r \mx{B}_c^T.
\end{align*}
When the features are generated from a kernel function, even a linear link function permits non-linear relationship between the features and the factor matrices.
\item \textbf{General regression models}: when the relationship between the features and the variable of interest is not linear, any off-the-shelf regression methods can be plugged in to search for a non-linear link function.
\end{itemize}

\subsubsection*{The choice of the optimization problem (\ref{eq:general_optimization_problem})}

Notice that with individual row and column features, linear link functions and complete observations, (\ref{eq:general_optimization_problem}) becomes 
\begin{equation}\label{eq:NMF_bis}
\begin{aligned}
\min_{\mx{F}_r, \mx{F}_c}\quad & \|\mx{V} - (\mx{F}_r)_+ (\mx{F}_c)_+^T\|_F^2.
\end{aligned}
\end{equation}
This is equivalent to the classical NMF problem,
\begin{equation}\label{eq:NMF}
\begin{aligned}
\min_{\mx{F}_r, \mx{F}_c}\quad & \|\mx{V} - \mx{F}_r (\mx{F}_c)^T\|_F^2\\
 \text{s.t.} \quad & \mx{F}_r \geq \mx{0}, \quad \mx{F}_c \geq \mx{0},
\end{aligned}
\end{equation}
in the sense that for any solution~$(\mx{E}_r, \mx{E}_c)$ to (\ref{eq:NMF_bis}),~$((\mx{E}_r)_+, (\mx{E}_c)_+)$ is a solution to (\ref{eq:NMF}).

A more immediate generalization of (\ref{eq:NMF}) to include exogenous variables would be in the form
\begin{equation}\label{eq:general_optimization_problem_naive}
\begin{aligned}
\min_{\mx{V},f_r \in F_r^k, f_c \in F_c^k}\quad & \|\mx{V} - f_r(\mx{X}_r) (f_c(\mx{X}_c))^T\|_F^2\\
 \text{s.t.}\quad & \operator{A}(\mx{V}) = \mx{b}, \quad \mx{V} \geq \mx{0}, \\
 & f_r(\mx{X}_r) \geq \mx{0}, \quad f_c(\mx{X}_c) \geq \mx{0}.
\end{aligned}
\end{equation}
Solving (\ref{eq:general_optimization_problem_naive}) would involve identifying the subset of~$F_r$ and~$F_c$ that only produce nonnegative value on the row and column features, which could be difficult.

\begin{table*}[!ht]
\centering
\caption{Classification of matrix factorization with side information by the mask, the link function, and the features included as side information, with some problems previously addressed in the literature.\label{table:classification}}
{\small
\begin{tabular}{lrp{2.3cm}p{2.3cm}p{2.3cm}p{2.3cm}}
\toprule
 & \textbf{Link function} & \multicolumn{3}{c}{Linear} & Other regression methods \\ 
\cmidrule(lr){2-2} \cmidrule(lr){3-5} \cmidrule(lr){6-6}
 & \textbf{Features} & Identity & General numeric features &  Kernel features & General numeric features \\ 
\midrule
\multirow{5}{*}{\rotatebox[origin=c]{90}{\textbf{Mask}}} & Identity  & Matrix factorization  & Reduced-regression rank \cite{velu_multivariate_2013,bunea_joint_2012,chen_sparse_2012} & Multiple kernel learning \cite{kekatos_electricity_2014} & Nonparametric RRR \cite{foygel_nonparametric_2012} \\
\cmidrule(lr){2-6}
& Matrix completion  & Matrix completion \cite{candes_exact_2009} & IMC\cite{agarwal_regression-based_2009,jain_provable_2013,xu_speedup_2013,chiang_matrix_2015} & GRMF \cite{abernethy_new_2009,rao_collaborative_2015,si_goal-directed_2016} &  \\
\cmidrule(lr){2-6}
& Rank-one projections  & \cite{cai_rop:_2015,zuk_low-rank_2015} &  &  &  \\
\cmidrule(lr){2-6}
& Temporal aggregates  & \cite{mei_nonnegative_2017} &  &  &  \\
\cmidrule(lr){2-6}
& General masks  & Matrix recovery \cite{candes_tight_2011,recht_guaranteed_2010} &  &  &  \\
\bottomrule
\end{tabular}
}
\end{table*}

By using (\ref{eq:general_optimization_problem}), we actually shifted the search space~$F_r$ and~$F_c$ to~$(F_r)_+$ and~$(F_c)_+$ which consists of composing all functions of~$F_r$ and~$F_c$ with the ramp function (thresholding at 0). In a word, (\ref{eq:general_optimization_problem}) is equivalent to 
\begin{equation}\label{eq:general_optimization_problem_+}
\begin{aligned}
\min_{\mx{V},f_r \in (F_r)_+^k, f_c \in (F_c)_+^k}\quad & \|\mx{V} - f_r(\mx{X}_r) (f_c(\mx{X}_c))^T\|_F^2\\
 \text{s.t.}\quad & \operator{A}(\mx{V}) = \mx{b}, \quad \mx{V} \geq \mx{0}.
\end{aligned}
\end{equation}
Problem (\ref{eq:general_optimization_problem_+}) also helps us to reason on the identifiability of (\ref{eq:general_optimization_problem}).
In a sense, (\ref{eq:general_optimization_problem}) is not well-identified: two distinct elements in~$F_c^k \times F_r^k$ have the same evaluation value of the objective function, if they only differ on their negative parts.
In fact, this does not affect the interpretation of the model, because these distinct elements correspond to the same element in~$(F_r)_+^k \times (F_c)_+^k$.
Since we are only going to use the positive parts of the function both in recovery and prediction, this becomes a parameterization choice which has no consequence on the applications.

As a comparison, we also propose an algorithm for the following optimization problem in Section \ref{sec:algorithm}:
\begin{equation}\label{eq:general_optimization_problem2}
\begin{aligned}
\min_{f_r \in F_r^k, f_c \in F_c^k}\quad & \|\mx{b} - \operator{A}((f_r(\mx{X}_r))_+ (f_c(\mx{X}_c))_+^T)\|_2^2.
\end{aligned}
\end{equation}
Instead of minimizing the low-rank approximation error for a matrix that matches the data as a linear matrix equation in (\ref{eq:general_optimization_problem}), (\ref{eq:general_optimization_problem2}) minimizes the sampling error of an exactly low-rank matrix.
Both have been studied in the literature.
For example, objective functions similar to (\ref{eq:general_optimization_problem}) have been considered in \cite{recht_guaranteed_2010}, and ones similar to (\ref{eq:general_optimization_problem2}) in \cite{roughan_spatio-temporal_2012}.
We will see in both Sections~\ref{sec:algorithm} and~\ref{sec:experiments} that~(\ref{eq:general_optimization_problem}) has a better performance than (\ref{eq:general_optimization_problem2}).

\subsection{Prior works}

Table \ref{table:classification} shows a taxonomy of matrix factorization models with side information, by the mask, the link function and the features used as side information.
A number of problems in this taxonomy has been addressed in established literature.

There is an abundant literature that studies the general matrix factorization problem under various measurement operators, 
when no additional information is provided 
(see \cite{rohde_estimation_2011,candes_exact_2009,recht_guaranteed_2010,zuk_low-rank_2015} for various masks considered, and \cite{bhojanapalli_global_2016} for a recent global convergence result with RIP measurements).
The NMF with general linear measurements is studied in various applications \cite{roughan_spatio-temporal_2012,pnevmatikakis_sparse_2013,mei_nonnegative_2017}.

On the other hand, with complete observations, the multiple regression problem taking into account the low-rank structure of the (multi-dimensional) variable of interest is known as reduced-rank regression.
This approach was first developed very early (see \cite{velu_multivariate_2013} for a review).
Recent developments on rank selection \cite{bunea_joint_2012}, adaptive estimation procedures \cite{chen_reduced_2013}, using non-parametric link function \cite{foygel_nonparametric_2012}, often draw the parallel between reduced-rank regression and the matrix completion problem.
However, measurement operators other than complete observations or matrix completion are rarely considered in this community.

Building on theoretical boundaries on matrix completion, the authors of \cite{jain_provable_2013,xu_speedup_2013,chiang_matrix_2015} showed that by providing side information (the matrix~$\mx{X}$), the number of measurements needed for exact matrix completion can be reduced. 
Moreover, the number of measurements necessary for successful matrix completion can be quantified by measuring the quality of the side information \cite{chiang_matrix_2015}.

Collaborative filtering with side information has received much attention from practitioners and academic researchers alike, for its huge impact in e-commerce and web applications \cite{agarwal_regression-based_2009,jain_provable_2013,xu_speedup_2013,chiang_matrix_2015,rao_collaborative_2015,si_goal-directed_2016}.
One of the first methods for including side information in collaborative filtering systems (matrix completion masks) was proposed by \cite{abernethy_new_2009}. 
The authors generalized collaborative filtering into a operator estimation problem. 
This method allows more general feature spaces than a numerical matrix, by applying a kernel function to side information. 
\cite{si_goal-directed_2016} proposed choosing the kernel function based on the goal of the application. 
\cite{kekatos_electricity_2014} applied the kernel-based collaborative filtering framework to electricity price forecasting. 
Their kernel choice is determined by multi-kernel learning methods.

To the best of our knowledge, matrix factorization (nonnegative or not) with side information, from general linear measurements has rarely been considered, nor is general non-linear functions other than with features obtained from kernels.
This article aims at proposing a general approach which fills this gap.

\section{Identifiability of nonnegative matrix factorization with side information\label{sec:identifiability}}

Matrix factorization is not a well-identified problem: 
for one pair of factors~$(\mx{F}_r, \mx{F}_c)$, with~$\mx{V}^\ast = \mx{F}_r \mx{F}_c^T~$, 
any invertible matrix~$\mx{R}$ produces another pair of factors,~$(\mx{F}_r\mx{R}, \mx{F}_c (\mx{R}^{-1})^T)$, with~$(\mx{F}_r \mx{R}) (\mx{F}_c (\mx{R}^{-1})^T )^T = \mx{V}^\ast$. 
In order to address this identifiability problem, one has to introduce extra constraints on the factors. 

When the nonnegativity constraint is imposed on~$\mx{F}_r$ and~$\mx{F}_c$, however, it has been shown that sometimes the only invertible matrices that verify~$\mx{F}_r \mx{R} \geq 0$ and~$\mx{R}^{-1} \mx{F}_c \geq 0$ are the composition of a permutation matrix and a diagonal matrix with strictly positive diagonal elements. 
A nonnegative matrix factorization is said to be ``identified'' if the factors are unique up to permutation and scaling. 
The identifiability conditions for NMF are a hard problem, because it turns out that NMF identifiability is equivalent to conditions that are computationally difficult to check.
In this section, we review some known necessary and sufficient conditions for NMF identifiability in the literature, and develop a sufficient condition for NMF identifiability in the context of linear numerical features.

In order to simplify our theoretical analysis, we focus on the complete observation case in this section (every entry in~$\mx{V}^\ast$ is observed).
Without loss of generality, we derive the sufficient condition for row features. 
That is, we will derive conditions on~$\mx{V}^\ast\in \R_+^{n_1 \times n_2}$ and~$\mx{X}_r \in \R^{n_1\times d_1}$, so that the nonnegative matrix factorization~$\mx{V}^\ast = \mx{X}_r \mx{B}_r \mx{F}_c^T$, with~$\mx{X}_r \mx{B}_r \geq 0, \mx{F}_c  \geq 0$, is unique.
A generalization to column features can be easily obtained.
In this section, we assume that in addition to be of nonnegative rank~$k$, matrix~$\mx{V}^\ast$ is also exactly of \textit{rank}~$k$.
% $\mx{V}^\ast = \mx{F}_r \mx{F}_c^T$, with~$(\mx{F}_r, \mx{F}_c) \in \R_+^{n_1\times k} \times \R_+^{n_2\times k}$.

\subsection{Identifiability of NMF}

The authors of \cite{donoho_when_2003} and \cite{laurberg_theorems_2008} proposed two necessary and sufficient conditions for the factorization to be unique. 
Both conditions use the following geometric interpretation of NMF introduced by \cite{donoho_when_2003}.

Since~$\mx{V}^\ast = \mx{F}_r \mx{F}_c^T$,
the columns of~$\mx{V}^\ast$ are conical combinations of the columns in~$\mx{F}_r$. 
Formally,~$\cone{\mx{F}_r}$, the conical hull of the columns of~$\mx{F}_r$, is a polyhedral cone contained in the first orthant of~$\R^{n_1}$. 
As~$\mx{V}^\ast$ is of rank~$k$, the rank of~$\mx{F}_r$ is also~$k$.
This implies that the extreme rays (also called \emph{generators}) of~$\cone{\mx{F}_r}$ are exactly the columns of~$\mx{F}_r$, which are linearly independent.
$\cone{\mx{F}_r}$ is therefore 
\begin{itemize}
\item a \emph{simplicial} cone of~$k$ generators,
\item contained in~$\R_+^{n_1}$, 
\item containing all columns of~$\mx{V}^\ast$.
\end{itemize}

Inversely, if we take any cone~$\operator{F} \subseteq \R^{n_1}$ verifying these three conditions, and define a matrix~$\mx{F}$ whose columns are the~$k$ generators of~$\operator{F}$, there will be a nonnegative matrix~$\mx{G}$, so that~$\mx{V}^\ast = \mx{F} \mx{G}$.
The uniqueness of NMF is therefore equivalent to the uniqueness of simplicial cones of~$k$ generators contained in the first orthant of~$\R^{n_1}$ and containing all columns~$\mx{V}^\ast$. 

In \cite{laurberg_theorems_2008}, an equivalent geometric interpretation in~$\R^{k}$ is given in the following theorem:
\begin{theo}\cite{laurberg_theorems_2008}{\label{CNS:WH}}
A~$k$-dimensional NMF~$\mx{V}^\ast = \mx{F}_r \mx{F}_c$ of a rank-$k$ nonnegative matrix~$\mx{V}^\ast$ is unique if and only if the nonnegative orthant~$\R^k_+$ is the only simplicial cone~$\operator{A}$ with~$k$ extreme rays satisfying
\begin{align*}
\coneT{\mx{F}_r^T} \subseteq \operator{A} \subseteq \coneT{\mx{F}_c^{\star}}.
\end{align*}
\end{theo}

Despite the apparent simplicity of the theorem, the necessary and sufficient conditions are very difficult to check. 
Based on the theorem above, several sufficient conditions have been proposed. 
The most widely used condition is called the separability condition. 
Before introducing this condition (in its least restrictive version presented by \cite{laurberg_theorems_2008}), we need the following two definitions.

\begin{definition}[Separability]
Suppose $m \leq n$.
A nonnegative matrix~$\mx{M}\in\R_+^{m \times n}$ is said to be \emph{separable} if there is a~$m$-by-$m$ permutation matrix~$\Pi$ which verifies
\begin{align*}
\mx{M} = \Pi \begin{pmatrix}
\mx{D}_n\\
\mx{M}_{0}
\end{pmatrix},
\end{align*}
where~$\mx{D}_n$ is a~$n$-by-$n$ diagonal matrix with only strictly positive coefficients on the diagonal and zeros everywhere else, and the~$(m-n)$-by-$n$ matrix~$\mx{M}_{0}$ is a collection of the other~$m - n$ rows of~$\mx{M}$.
\end{definition}

\begin{definition}[Strongly Boundary Closeness]\label{def:facet}
A nonnegative matrix~$\mx{M}\in \R_+^{m \times n}$ is said to be \emph{strongly boundary close} if the following conditions are satisfied.
\begin{enumerate}
\item~$\mx{M}$ is \emph{boundary close}: for all~$i, j\in \{1,...,n\}, i \neq j$, there is a row~$\mx{m}$ in~$\mx{M}$ which satisfies~$m_i  = 0, m_j > 0$;
\item There is a permutation of~$\{1,..., n\}$ such that for all~$i\in \{1,..., n-1\}$, there are~$n-i$ rows~$\mx{m}^1,..., \mx{m}^{n-i}$ in~$\mx{M}$ which satisfy \begin{enumerate}
\item~$m^j_i = 0, \sum_{s = i+1}^n m^j_i > 0$ for all~$j \in \{1,...,n-i\}$; 
\item the square matrix~$(m^j_s)_{1 \leq j \leq n-i, i+1 \leq s \leq n}$ is of full rank~$(n - i)$.
\end{enumerate}
\end{enumerate}
\end{definition}

Strongly boundary closeness demands, \textit{modulo} a permutation in~$\{1,...,n\}$, that for each~$1 \leq i \leq n - 1$, there are~$n-i$ rows~$\mx{m}^1,..., \mx{m}^{n-i}$ of~$\mx{M}$ that have the following form,
\begin{equation}\label{matrix:M^i}
\begin{aligned}
 &\begin{pmatrix} 
 &\mx{m}^1 \\ 
 & \colon \\ 
 &\mx{m}^{n-i} 
 \end{pmatrix}^T = \\
 &\begin{pmatrix}
  \colon & \ddots &\colon \\
   0 & \cdots  & 0 \\
   m^1_{i+1} & \cdots & m^{n-i}_{i+1}\\
  \colon & \ddots & \colon\\
   m^1_{n} & \cdots & m^{n-i}_{n}\\
 \end{pmatrix}%
 \begin{matrix}%
 (i-1)\text{ first rows}\\
 i\text{-th row is all zero}\\
 \coolrightbrace{m^{n-i}_{i+1} \\ \colon \\ m^{n-i}_{n} } {\text{$(n-i)\times(n-i)$ full rank matrix}}\\
 \end{matrix}
\end{aligned}
\end{equation}
These row vectors,~$\mx{m}^1,..., \mx{m}^{n-i}$, all have 0 on the~$i$-th element, and its lower square matrix of is of full rank. 
There are therefore enough linearly independent points on each~$n-1$-dimensional facet~$\R_+^n$, which shows that~$\cone{\mx{M}^T}$ is somewhat maximal in~$\R_+^n$.

The following was proved in \cite{laurberg_theorems_2008}:
\begin{theo}\cite{laurberg_theorems_2008}
If~$\mx{F}_r$ is strongly boundary close, then the only simplicial cone with~$k$ generators in~$\mathbb{R}^k_+$ containing~$\text{cone}(\mx{F}_r^T)$ is~$\mathbb{R}^k_+$. 
Moreover, if~$\mx{F}_c$ is separable, then~$\mx{V}^\ast = \mx{F}_r \mx{F}_c^T$ is the unique NMF of~$\mx{V}^\ast$ up to permutation and scaling.
\end{theo}

\subsection{Identifiability with side information}

The NMF with linear row features,~$\mx{V}^\ast = \mx{X}_r \mx{B}_r \mx{F}_c^T~$, is said to be \emph{unique}, if for all matrix pairs~$(\tilde{\mx{B}}_r, \tilde{\mx{F}}_c) \in \R^{d_1\times k} \times \R^{n_2\times k}~$ that verifies 
\begin{align*}
\mx{X}_r\tilde{\mx{B}}_r \geq 0, \quad \tilde{\mx{F}}_c \geq 0,  \quad \mx{V}^\ast = \mx{X}_r\tilde{\mx{B}}_r \tilde{\mx{F}}_c,
\end{align*}
we have~$\tilde{\mx{B}} = \mx{B}_r,$~$ \tilde{\mx{F}}_c = \mx{F}_c$ up to permutation of columns and scaling.

For a given full-rank matrix~$\mx{X} \in \R^{n_1 \times d_1}$, consider the following two sets of matrices:
\begin{align*}
E = \{ \mx{M} \in \R_+^{n_1 \times k} | & \text{The columns of } \mx{M} \text{ are} \\
& \text{strongly boundary close} \};\\
F(\mx{X}) = \{ \mx{M} \in \R_+^{n_1 \times k} |  & \text{rank}(\mx{M}) = k, \text{span}(\mx{M}) \in \text{span}(\mx{X})\}.
\end{align*}

\begin{theo}
If~$E \bigcap F(\mx{X}_r) \neq \varnothing$, and~$\mx{B}_r \in (\mx{X}_r^T \mx{X}_r)^{-1}\mx{X}_r^T (E \bigcap F(\mx{X}_r))~$, and~$\mx{F}_c$ is separable, then the factorization~$\mx{V}^\ast = \mx{X}_r \mx{B}_r \mx{F}_c^T~$ is unique.
\end{theo}

\begin{proof}
Notice that for~$\mx{B}_r \in (\mx{X}_r^T \mx{X}_r)^{-1}\mx{X}_r^T (E \bigcap F(\mx{X}_r))$, the nonnegative matrix~$\mx{X}_r\mx{B}_r$ is strongly boundary close. 
The factorization~$(\mx{X}_r\mx{B}_r, \mx{F}_c)$ is therefore unique. 
The model identifiability follows immediately, since~$\mx{X}_r$ is of full rank.
\end{proof}

\subsubsection*{Example of~$\mx{X}_r$ that verifies~$E \bigcap F(\mx{X}_r) \neq \varnothing$}

For this theorem to have practical consequences, one needs to find appropriate row features so that~$E \bigcap F(\mx{X}_r) \neq \varnothing$.

Here we provide a family of matrices~$\mx{X}_r$ so that~$E \bigcap F(\mx{X}_r) \neq \varnothing$.

% Here we construct an example for which the vectors mentioned in the second point of Definition \ref{def:facet} (the rows in~$\mx{F}_r$ as linearly independant points on each~$K-1$-dimensional facet of~$\R_+^\ast$) are distinct for different values of~$ 1\leq k \leq K-1$. 
% This means there are~$K (K-1)/2$ rows on~$\mx{X}$ on which some restrictions are to be imposed. 
% The first point in Definition \ref{def:facet} imposes an additional row on which the~$K$-th entry must be zero. 
% This amounts to~$n(K) \equiv K (K-1)/2 + 1$ special rows in~$\mx{X}$. 
% Suppose, without loss of generality, that these are the~$n(K)$ first rows of~$\mx{X}$. 
% Suppose~$D = n(K)$.

With a fixed~$k\geq 2$, suppose that~$\mx{X}_r$ has~$k(k - 1)/2$ columns, and at least~$k(k - 1)/2$ rows, with the first~$k(k - 1)/2 + 1$ rows defined as the following: 
\begin{itemize}
\item the first row and column have 0 on the first entry and positive entries elsewhere;
\item for~$2 \leq i \leq k$, ~$\mx{X}_r$ has strictly positive entries on the first~$((i - 1)(i - 2)/2 + 1)$ columns, from Row~$(i - 1)(i - 2)/2 + 3$ to Row~$(i - 1)(i - 2)/2 + 1 + i$, and zero entries everywhere else.
These~$(k-1)$ rows are linearly independent.
\end{itemize}

Then we have~$E \bigcap F(\mx{X}_r) \neq \varnothing$, because the following~$k(k - 1)/2$-by-$k$ matrix~$\mx{B}_r$ is in this set:
\begin{itemize}
\item for~$1 \leq i \leq k$,~$\mx{B}_K^\ast$ has~$i$ consecutive strictly positive entries on the~$i$-th column, between Row~$i(i-1)/2 + 1$ and Row~$i(i-1)/2 + i$.
\end{itemize}

The following matrices instantiate the case of~$k = 4$:
\begin{align*}
&
\mx{B}_r = \begin{pmatrix}
1 & 0 & 0 & 0 \\ 
  0 & 1 & 0 & 0 \\ 
  0 & 0 & 1 & 0 \\ 
  0 & 0 & 1 & 0 \\ 
  0 & 0 & 0 & 1 \\ 
  0 & 0 & 0 & 1 \\ 
  0 & 0 & 0 & 1
\end{pmatrix}, \\
&\mx{X}_r = \begin{pmatrix}
0 & 5 & 14 & 7 & 9 & 15 & 13 \\ 
  10 & 0 & 0 & 0 & 0 & 0 & 0 \\ 
  4 & 5 & 0 & 0 & 0 & 0 & 0 \\ 
  12 & 4 & 0 & 0 & 0 & 0 & 0 \\ 
  10 & 7 & 10 & 7 & 0 & 0 & 0 \\ 
  13 & 10 & 12 & 9 & 0 & 0 & 0 \\ 
  12 & 10 & 16 & 8 & 0 & 0 & 0 \\ 
  \colon & \colon & \colon & \colon & \colon & \colon & \colon
\end{pmatrix}, \\
&\mx{F}_r = \begin{pmatrix}
0 & 5 & 21 & 37 \\ 
  10 & 0 & 0 & 0 \\ 
  4 & 5 & 0 & 0 \\ 
  12 & 4 & 0 & 0 \\ 
  10 & 7 & 17 & 0 \\ 
  13 & 10 & 21 & 0 \\ 
  12 & 10 & 24 & 0 \\ 
  \colon & \colon & \colon & \colon
\end{pmatrix}.
\end{align*}

If~$E \bigcap F(\mx{X}) \neq \varnothing$, for any invertible matrix~$\mx{R} \in \R^{K\times K}$, ~$E \bigcap F(\mx{X}\mx{R}) \neq \varnothing$.

\section{HALSX algorithm \label{sec:algorithm}}

In this section, we propose \emph{Hierarchical Alternating Least Squares with eXogeneous variables} (HALSX), a general algorithm to estimate the nonnegative matrix factorization problem with side information, from linear measurement, by solving (\ref{eq:general_optimization_problem}).  
It is an extension to a popular NMF algorithm: Hierarchical Alternating Least Squares (HALS) (see \cite{cichocki_hierarchical_2007,kim_algorithms_2014}).

Before discussing HALSX, we will first present a result on the local convergence of Gauss-Seidel algorithms. 
This result guarantees that any legitimate limiting points generated by HALSX are critical points of (\ref{eq:general_optimization_problem}).

While presenting specific methods to estimate link functions, we will only discuss row features,
as a generalization to column features is immediate.

\subsection{Relaxation of convexity assumption for the convergence of Gauss-Seidel algorithm}

To show the local convergence of HALSX algorithm, we first extend a classical result concerning block nonlinear Gauss-Seidel algorithm \cite[Proposition 4]{grippo_convergence_2000}.

Consider the minimization problem,
\begin{equation}\label{eq:gauss_seidel problem}
\begin{aligned}
\min\quad & g(x)\\
 \text{s.t.}\quad & x \in X = X_1 \times X_2 \times ... \times X_m \subseteq \R^n,
\end{aligned}
\end{equation}
where~$g$ is a continuously differentiable real-valued function, 
and the feasible set~$X$ is the Cartesian product of closed, nonempty and convex subsets~$X_i \subset \R^{n_i}$, for~$1 \leq i \leq m$, with~$\sum_i n_i = n$.
Suppose that the global minimum is reached at a point in $X$. 
The~$m$-block Gauss-Seidel algorithm is defined as Algorithm~\ref{algo:GS}.

\begin{algorithm}
\caption{Gauss-Seidel algorithm}
\label{algo:GS}
\begin{algorithmic}
\STATE Initialize~$x^0 \in X, t = 0$\;
\WHILE{Stopping criterion is not satisfied}
\FOR{$i = 1,2,...,m$}
\STATE Calculate~$x_i^{t+1} = \arg \min_{y_i \in X_i} g(x_1^{t+1}, ..., y_i, ..., x_m^t)$\;
\ENDFOR
\STATE  Set~$x^{t + 1} = (x_1^{t+1}, ..., x_m^{t+1})$\;
\STATE ~$t = t+1$\;
\ENDWHILE
\end{algorithmic}
\end{algorithm}

Define formally the notion of component-wise quasi-convexity.
\begin{definition}
Let~$i \in \{1,2,...,m \}$. 
The function~$g$ is \textit{quasi-convex} with respect to the~$i$-th component on~$X$ if for every~$x \in X$ and~$y_i \in X_i$, 
we have
\begin{align*}
& g(x_1, x_2,... t x_i + (1 - t) y_i, ..., x_m) \\
\leq &\max\{g(x), g(x_1, x_2, ..., y_i, ... x_m) \}
\end{align*}
for all~$t \in [0, 1]$. 
$g$ is said to be strictly quasi-convex with respect to the~$i$-th component, 
if with the additional assumption that~$y_i \neq x_i$, 
we have
\begin{align*}
& g(x_1, x_2,... t x_i + (1 - t) y_i, ..., x_m)\\ 
< & \max\{g(x), g(x_1, x_2, ..., y_i, ... x_m) \}
\end{align*}
for all~$t \in ]0, 1[$.
\end{definition}

It has been shown that if~$g$ is strictly quasi-convex with respect to the first~$m - 2$ blocks of components on~$X$, 
then a limiting point produced by a Gauss-Seidel algorithm is a critical point \cite{grippo_convergence_2000}. 

This result is not directly applicable for the HALS algorithm. 
Typically, if~$\mx{f}_{c,i}$, the~$i-th$ column of~$\mx{F}_c$, is identically zero, the loss function is completely flat respect to~$\mx{f}_{c,i}$, the~$i$-th column of~$\mx{F}_r$. 
Therefore the loss function is not strictly quasi-convex. 
In order to avoid this scenario, \cite{kim_algorithms_2014} suggests thresholding at a small positive number~$\epsilon$ instead of at 0, when updating each column of the factor matrices. 

In fact the convexity assumption of \cite{grippo_convergence_2000} can be slightly relaxed to directly apply to HALS, as demonstrated by the following proposition.
\begin{theo}\label{theo:GS}
Suppose that the function~$g$ is quasi-convex with respect to~$x_i$ on~$X$, for~$i = 1, ..., m -2$. 
Suppose that some limit points~$\bar{x}$ of the sequence~$\{x^t\}_{(t \in \N)}$ verify that~$g$ is strictly quasi-convex with respect to~$x_i$ on the product set~$\{\bar{x}_1\}\times \{\bar{x}_2\} \times ... \times X_i \times, ... \times \{\bar{x}_m\}$, for~$i = 1, ..., m -2$. 
Then every such limiting point is a critical point of Problem (\ref{algo:GS}).
\end{theo}
Compared to the result of \cite{grippo_convergence_2000}, this shows that the strict convexity with respect to one block does not have to hold universally for feasible regions of other blocks. 
It only needs to hold at the limiting point.

This theorem can be established following the proof of Proposition 5 of \cite{grippo_convergence_2000}, 
using the following lemma.
\begin{lemma}
Suppose that the function~$g$ is quasi-convex with respect to~$x_i$ on~$X$, for some~$i \in \{1,...,m\}$. 
Suppose that some limit points~$\bar{y}$ of~$\{y^t\}$ verify that~$g$ is strictly quasi-convex with respect to~$x_i$ on~$\{\bar{y}_1\}\times \{\bar{y}_2\} \times ... \times X_i \times, ... \times \{\bar{y}_m\}$. Let~$\{v^t\}$ be a sequence of vectors defined as follows:
\begin{align*}
v_j^t = \left\{
\begin{array}{ll}
y_j^t \quad & \text{ if } j \neq i,\\
\arg \min_{z_i \in X_i} g(y_1^{t}, ..., z_i, ..., y_m^t)\quad & \text{ if } j = i.
\end{array}
\right.
\end{align*}
Then, if ~$\lim_{t \rightarrow +\infty} g(y^t) - g(v^t) = 0$, we have~$\lim_{t \rightarrow +\infty} || v_i^t - y_i^t || = 0$. That is~$\lim_{t \rightarrow +\infty} || v^t - y^t || = 0$.
\end{lemma}
\begin{proof} 
(The proof of the lemma is based on \cite{bertsekas_parallel_1989}.)

Suppose on the contrary that~$|| v_i^t - y_i^t ||$ does not converge to 0. 
Define~$\tau_k = || v_i^t - y_i^t ||$. 
Restricting to a subsequence, we can obtain that~$\tau_k \geq \tau_0 > 0.$ 
Define~$s^t = \frac{v^t - y^t}{\tau_k}$. 
Notice that~$\{s^t\}$ is of unit norm, and~$v^t = y^t + \tau_k s^t$. Since~$\{s^t\}$ is on the unit sphere, 
it has a converging subsequence. 
By restricting to a subsequence again, we could suppose that~$\{s^t\}$ converges to~$\bar{s}$.

For all~$\epsilon \in [0, 1]$, we have~$ 0 \leq \epsilon \tau_0 \leq \tau_k$, which implies~$y^t + \epsilon \tau_0 s^t  \in X$ is on the segment~$[y^t, v^t]$. 
This segment has strictly positive dimension in the subspace corresponding to~$X_i$.

By the definition of~$\{v^t \}$,~$g(v^t) \leq g(y_1^{t}, ..., z_i, ..., y_m^t)$, for all~$t$, and for all~$z_i \in X_i$. In particular,
\begin{align*}
g(v^t) \leq g( y^t + \epsilon \tau_0 s^t).
\end{align*}

By quasi-convexity of~$g$ on~$X$, 
\begin{align*}
g(y^t + \epsilon \tau_0 s^t) \leq \max\{g(y^t), g(v^t)\} = g(y^t).
\end{align*}

Taking the limit when~$t$ converges to~$+\infty$ on both equalities, we obtain
\begin{align*}
& g(\bar{y}) = \lim_{t \rightarrow +\infty} g(v^t) \leq  \lim_{t \rightarrow +\infty} g(y^t + \epsilon \tau_0 s^t)\\
=  & g(\bar{y} + \epsilon \tau_0 \bar{s}) \leq \lim_{t \rightarrow +\infty} g(y^t) = g(\bar{y}).
\end{align*}
In other words,~$g(\bar{y} + \epsilon \tau_0 \bar{s}) = g(\bar{y})$,~$\forall \epsilon \in [0, 1]$, 
which contradicts the strict quasi-convexity of~$g$ on~$\{\bar{y}_1\}\times \{\bar{y}_2\} \times ... \times X_i \times, ... \times \{\bar{y}_m\}$.
\end{proof}

\subsection{HALSX algorithm}

To solve (\ref{eq:general_optimization_problem}), we propose HALSX (Algorithm \ref{algo:halsx}).
When complete observations are available, the feature matrices are identity matrices, and when only linear functions are allowed as link functions, Algorithm \ref{algo:halsx} is equivalent to HALS \cite{cichocki_hierarchical_2007}.

\begin{algorithm}
\caption{Hierarchical Alternating Least Squares with eXogeneous variables for NMF (HALSX)}
\label{algo:halsx}
\begin{algorithmic}
\REQUIRE Measurement operator~$\operator{A}$, measurements~$\mx{b}$, features~$\mx{X}_r$ and~$\mx{X}_c$, functional spaces~$F_r$ and~$F_c$ in which to search the link functions, and~$1\leq k \leq \min\{n_1, n_2\}$.\;
\STATE Initialize~$\mx{F}_r^0, \mx{F}_c^0 \geq 0,  t = 0$\;
\WHILE{Stopping criterion is not satisfied}
\STATE ~$\mx{V}^{t} = \arg \min_{\mx{V} | \operator{A}(\mx{V}) = \mx{b},  \mx{V}\geq 0}\|\mx{V} - \mx{F}_r^t (\mx{F}_c^t)^T \|_F^2$\;
\STATE~$\mx{R}^t = \mx{V}^t - \mx{F}_r^t (\mx{F}_c^t)^T$
\FOR{$i = 1,2,...,k$}
\STATE~$\mx{R}_i^t = \mx{R}^t + \mx{f}_{r,i}^t (\mx{f}_{c,i}^t)^T~$
\STATE Calculate~$f_{r,i}^{t+1} = \arg \min_{f \in F_r} \| \mx{R}_i^t - f(\mx{X}_r) (\mx{f}_{c,i}^t)^T \|_F^2$\;
\STATE~$\mx{f}_{r,i}^{t+1} = \max(0, f_{r,i}^{t+1}(\mx{X}_r))$
\STATE~$\mx{R}^t = \mx{R}_i^t - \mx{f}_{r,i}^{t+1} (\mx{f}_{c,i}^t)^T~$
\ENDFOR
\FOR{$i = 1,2,...,k$}
\STATE~$\mx{R}_i^t = \mx{R}^t + \mx{f}_{r,i}^{t+1} (\mx{f}_{c,i}^t)^T~$
\STATE Calculate~$f_{c,i}^{t+1} = \arg \min_{f \in F_c} \| \mx{R}_i^t - \mx{f}_{r,i}^{t+1} f(\mx{X}_c)^T \|_F^2$\;
\STATE~$\mx{f}_{c,i}^{t+1} = \max(0, f_{c,i}^{t+1}(\mx{X}_c))$
\STATE~$\mx{R}^{t} = \mx{R}_i^t - \mx{f}_{r,i}^{t+1} (\mx{f}_{c,i}^{t+1})^T~$
\ENDFOR
\STATE ~$t = t+1$\;
\ENDWHILE
\RETURN~$\mx{V}^t =  \arg \min_{\mx{V} | \operator{A}(\mx{V}) = \mx{b},  \mx{V}\geq 0}\|\mx{V} - \mx{F}_r^t (\mx{F}_c^t)^T \|_F^2$, \\
~$\mx{F}_r^t \in \mathbb{R}_+^{n_1 \times k}, f_{r,1}^t,..., f_{r,k}^t \in F_r$,\\
~$\mx{F}_c^t \in \mathbb{R}_+^{n_2\times k}, f_{c,1}^t,..., f_{c,k}^t \in F_c.$
\end{algorithmic}
\end{algorithm}

From Theorem \ref{theo:GS}, one deduces that every full-rank limiting point produced by the popular HALS algorithm is a critical point.

In this algorithm, at each elementary update step, we first look for a link function which minimizes the quadratic error, without concerning ourselves with its nonnegativity. 
The obtained evaluation of the minimizer function is then thresholded at~$\mx{0}$ to update the factors.

To show the convergence of HALSX algorithm, we need to assure that for some functional spaces~$F_r$ and~$F_c$, such an update solves a corresponding subproblem of (\ref{eq:general_optimization_problem}). 
To do this, we will use the following proposition:
\begin{proposition}\label{prop:jacobian}
Suppose that~$\mx{R} \in \R^{n_1 \times n_2}$,~$\mx{f}_c \in \R_+^{n_2}$ are not identically equal to zero, and $g : \R^d \rightarrow \R^{n_1}$, with~$d \geq n_1$, is a convex differentiable function. Suppose
\begin{align*}
\mx{\theta}^\ast \in \arg \min_{\mx{\theta} \in \R^d} \| \mx{R} - g(\mx{\theta}) (\mx{f}_c)^T \|_F^2.
\end{align*}
If~$\nabla g_{\mx{\theta}^\ast}$, the Jacobian matrix of~$g$ at~$\mx{\theta}^\ast$, is of rank~$n_1$,
then $\mx{\theta}^\ast$
is also a solution to 
\begin{align}\label{eq:convex_positivity}
\min_{\mx{\theta} \in \R^d} \| \mx{R} - (g(\mx{\theta}))_+ (\mx{f}_c)^T \|_F^2.
\end{align}
\end{proposition}

\begin{proof}
Take~$\mx{R} \in \R^{n_1 \times n_2}$,~$\mx{f}_c \in \R_+^{n_2}$ not identically equal to zero. 
We will note by~$L$ the loss function, so that~$L(\mx{f}) = \| \mx{R} - (\mx{f})_+ (\mx{f}_c)^T \|_F^2$ for all~$\mx{f} \in \R^{n_1}$.
The function~$L$ is convex.

Problem (\ref{eq:convex_positivity}), which can be rewritten as 
\begin{align*}
\min_{\mx{\theta} \in \R^d} L(g(\mx{\theta})),
\end{align*}
is also convex.
The subgradient of the composition function~$L \circ g$ at~$\mx{\theta} \in \R^d$ is simply obtained by multiplying~$\nabla g_{\mx{\theta}}$, the Jacobian matrix of~$g$ at~$\mx{\theta}$, to each element of~$\partial L_{g(\mx{\theta})}$, or~$\partial L_{g(\mx{\theta})} \equiv \nabla g_{\mx{\theta}} \partial L_{g(\mx{\theta})} = \{\nabla g_{\mx{\theta}}\mx{y} | \mx{y} \in  \partial L_{g(\mx{\theta})}\}$.
Therefore~$\forall \mx{\theta} \in \R^{d}$,~$\mx{\theta}$ is a minimizer of (\ref{eq:convex_positivity}), if and only if~$\mx{0} \in \nabla g_{\mx{\theta}}\partial L_{g(\mx{\theta})}$.

Since
\begin{align*}
\mx{\theta}^\ast \in \arg \min_{\mx{\theta} \in \R^d} \| \mx{R} - g(\mx{\theta}) (\mx{f}_c)^T \|_F^2,
\end{align*}
is a minimizer of a smooth convex problem,
\begin{align*}
\frac{\partial}{\partial \mx{\theta}} \| \mx{R} - g(\mx{\theta}) (\mx{f}_c)^T \|_F^2 (\mx{\theta}^\ast) = \nabla g_{\mx{\theta}^\ast} (\mx{R} - g(\mx{\theta}^\ast) (\mx{f}_c)^T) \mx{f}_c = \mx{0}.
\end{align*}
This means~$(\mx{R} - g(\mx{\theta}^\ast) (\mx{f}_c)^T) \mx{f}_c = \mx{0}$, because~$\nabla g_{\mx{\theta}^\ast}$ is of full rank.
Consequently 
\begin{align*}
g(\mx{\theta}^\ast) = \frac{1}{\|\mx{f}_c\|_2^2} \mx{R} \mx{f}_c.
\end{align*}
It has been shown in NMF literature (for example \cite[Theorem 2]{kim_algorithms_2014}) that,
\begin{align*}
(g(\mx{\theta}^\ast))_+ = \arg\min_{\mx{f} \in \R_+^{n_1}} \| \mx{R} - \mx{f} (\mx{f}_c)^T \|_F^2.
\end{align*}
This is equivalent to
\begin{align*}
g(\mx{\theta}^\ast) \in \arg \min_{\mx{f} \in \R^{n_1}} L(\mx{f}),
\end{align*}
or 
\begin{align*}
\mx{0} \in \partial L_{g(\mx{\theta}^\ast)}.
\end{align*}
We therefore conclude with~$\mx{0} \in \nabla g_{\mx{\theta}^\ast}\partial L_{g(\mx{\theta}^\ast)}$.
\end{proof}

In many regression methods, even when a non-linear transformation is applied to the data, the regression function is linear in its parameters.
A non-exhaustive list of methods include linear regression~$(g(\mx{\theta}) = \mx{X}_r \mx{\theta})$, spline regression~$(g(\mx{\theta}) = \phi(\mx{X}_r) \mx{\theta})$, or support vector regression (SVR)~$(g(\mx{\theta}) = K(\mx{X}_r, \mx{X}_r) \mx{\theta})$.
In this case,~$g$ has a constant Jacobian matrix.
In the case of linear and spline regression, the Jacobian matrix is of rank $n_1$ if there are no less features than examples. 
For SVR, this is true for any positive definite kernels.
This allows us to apply the previous lemma to each column update step of Algorithm \ref{algo:halsx}.
By calculating~$f_{r,i}^{t+1} = \arg \min_{f \in F_r} \| \mx{R}^t - f(\mx{X}_r) (\mx{f}_{c,i}^t)^T \|_F^2$ at Step~$t$ for Column~$i$ in~$\mx{F}_r$, we actually have
\begin{align*}
f_{r,i}^{t+1} = \arg \min_{f \in F_r} \| \mx{R}^t - (f(\mx{X}_r))_+ (\mx{f}_{c,i}^t)^T \|_F^2.
\end{align*}
This shows that at each iteration, we solve the subproblems of (\ref{eq:general_optimization_problem}).

In these cases, by rewriting the functional space~$F_r$ and~$F_c$ in a parametric form, the search space is actually~$\R^{r_1}$ and~$\R^{r_2}$, for some~$r_1$ and~$r_2$.
\begin{proposition}
If~$n_1 \leq r_1$,~$n_2 \leq r_2$, and the regression functions are linear in parameters with a full-rank Jacobian matrix,
every full-rank limiting point produced by HALSX (Algorithm \ref{algo:halsx}) is a critical point of Problem (\ref{eq:general_optimization_problem}).
\end{proposition}

\begin{rem}
In order to extend this convergence result for Problem (\ref{eq:general_optimization_problem_naive}), one would need to ensure the obtained functions have non-negative values on the features. 
This could be done by alternating projection.
\end{rem}

\subsection{Designs and HALSX}

At each iteration of Algorithm \ref{algo:halsx}, we need to project the working matrix~$\mx{F}_r^t (\mx{F}_c^t)^T$ into the convex polytope defined by the measurements and nonnegativity: 
\begin{align} \label{eq:projection}
\mx{V}^{t} = \arg \min_{\mx{V} | \operator{A}(\mx{V}) = \mx{b},  \mx{V}\geq 0}\|\mx{V} - \mx{F}_r^t (\mx{F}_c^t)^T \|_F^2.
\end{align}

In general, the polytope projection can be obtained by alternating projection.
Namely, we can alternate between:
\begin{itemize}
\item~$\mx{V} = \mx{V} + \operator{A}^\dagger (\mx{b} - \operator{A}(\mx{V}))$;
\item~$ v_{i,j}  = \max(0, v_{i,j})$,
\end{itemize}
where~$\operator{A}^\dagger$ is the right pseudo-inverse of~$\operator{A}$, viewed as an~$N$-by-$n_1 n_2$ matrix.

For some measurement operators, there are efficient ways to solve (\ref{eq:projection}).
\begin{itemize}
\item Matrix completion mask: \\$ v_{i,j}  = \left\{ 
\begin{matrix}
& \alpha_l, \quad \text{if } \exists 1 \leq l \leq N, \mx{A}_l = \mx{e}_i \mx{e}_j^T;\\
& \max(0, v_{i,j}), \quad \text{if not.}
\end{matrix}
 \right.$
\item Temporal aggregate mask: simplex projection (see \cite{mei_nonnegative_2017} for details). 
\end{itemize}

\subsection{Linear HALSX}

In this section, we consider HALSX with numeric row features and linear row link functions.
That is, given~$\mx{X}_r$ and~$\mx{\alpha} = \operator{A}(\mx{V}^\ast)$, we need to solve
\begin{equation}\label{eq:linear_row_features}
\begin{aligned}
\min_{\mx{V}\in \R^{n_1 \times n_2},\mx{B}_r \in \R^{d_1 \times k}, \mx{F}_c \in \R^{n_2 \times k}}\quad & \|\mx{V} - (\mx{X}_r \mx{B}_r)_+ (\mx{F}_c)_+^T\|_F^2\\
 \text{s.t.}\quad & \operator{A}(\mx{V}) = \mx{b}, \quad \mx{V} \geq \mx{0},
\end{aligned}
\end{equation}

Following Algorithm \ref{algo:halsx}, we need to update the columns of~$\mx{B}_r$ at each iteration.
At the~$t$-th step, for~$1 \leq i \leq k$, we solve the subproblem
\begin{align*}
\arg \min_{\mx{b}_{r,i}} \| \mx{R}^t_i - \mx{X}_r \mx{b}_{r,i} (\mx{f}_{c,i}^t)^T \|_F^2,
\end{align*}
where~$\mx{R}^t_i = \mx{V}^t - \sum_{j = 1, j \neq i}^k \mx{X}_r\mx{b}_{r,l} \mx{f}_{c,l}^T.$
This minimization problem has a closed-form solution:
\begin{align*}
\mx{b}_{r,i}^{t+1} = \frac{1}{\| \mx{f}_{c,i}^t \|_2^2} (\mx{X}_r^T\mx{X}_r)^{-1}\mx{X}_r^T\mx{R}_i^t\mx{f}_{c,i}^t.
\end{align*}

In order to accelerate the numerical algorithm, a QR decomposition of~$\mx{X}_r = \mx{Q}\mx{R}$ is done before the iterations, where~$\mx{Q}$ is an orthogonal matrix, and~$\mx{R}$ is a square upper triangular matrix. 
When~$\mx{X}_r$ is of full rank,~$\mx{X}_r^T\mx{X}_r$ is invertible. 
We compute one time~$(\mx{X}_r^T\mx{X}_r)^{-1}\mx{X}_r = \mx{R}^{-1}\mx{Q}^T$ before the iterations, and use the result at each iteration.

\subsubsection*{Stopping criterion}

As in classical NMF algorithms, we will use the Karush–Kuhn–Tucker conditions (KKT) to provide a stopping criterion. 
The KKT conditions of (\ref{eq:linear_row_features}) are,
\begin{align*}
& \mx{V} \geq \mx{0}, \quad \mx{V} - (\mx{X}_r \mx{B}_r)_+ (\mx{F}_c)_+^T \geq \mx{0}, \quad \operator{A}(\mx{V}) = \mx{b}, \\
& \mx{V} - (\mx{X}_r \mx{B}_r)_+ (\mx{F}_c)_+^T \circ \mx{V} = \mx{0},\\
&  \nabla_{\mx{F}_c} \|\mx{V} - (\mx{X}_r \mx{B}_r)_+ (\mx{F}_c)_+^T\|_F^2 \ni  \mx{0}, \\
& \nabla_{\mx{B}_r} \|\mx{V} - (\mx{X}_r \mx{B}_r)_+ (\mx{F}_c)_+^T\|_F^2 \ni \mx{0},
\end{align*}
where~$\mx{A} \circ \mx{B}$ is the entry-wise product (Hadamard product) for~$\mx{A}, \mx{B}$ of the same dimension, and $\nabla_x f(x_0)$ is the subgradient of the function $f$ at point $x_0$, with respect to the variable $x$.
Note that $\mx{V} \geq \mx{0}$ and $\operator{A}(\mx{V}) = \mx{b}$ are always satisfied at the end of an iteration.

As 
$(\mx{X}_r \mx{B}_r)_+^T (\mx{V} - (\mx{X}_r \mx{B}_r)_+ (\mx{F}_c)_+^T) \circ \mx{\mathds{1}}_{(\mx{F}_c)> \mx{0}}$ 
and 
$\mx{X}_r^T ((\mx{V} - (\mx{X}_r \mx{B}_r)_+ (\mx{F}_c)_+^T) \circ \mx{\mathds{1}}_{\mx{X}_r \mx{B}_r> \mx{0}})$ 
are respectively in the subgradient with respect to $\mx{F}_c$ and $\mx{B}_r$, 
we will stop the algorithm when the norm of following vector
\begin{align*}
[ 
& vect((\mx{V} - (\mx{X}_r \mx{B}_r)_+ (\mx{F}_c)_+^T)_-)^T, \\
& vect(\mx{V} - (\mx{X}_r \mx{B}_r)_+ (\mx{F}_c)_+^T \circ \mx{V})^T, \\
& vect((\mx{X}_r \mx{B}_r)_+^T (\mx{V} - (\mx{X}_r \mx{B}_r)_+ (\mx{F}_c)_+^T) \circ \mx{\mathds{1}}_{(\mx{F}_c)> \mx{0}})^T, \\
& vect(\mx{X}_r^T ((\mx{V} - (\mx{X}_r \mx{B}_r)_+ (\mx{F}_c)_+^T) \circ \mx{\mathds{1}}_{\mx{X}_r \mx{B}_r> \mx{0}}))^T
],
\end{align*}
is smaller than $\epsilon$ times its initial value, with a small $\epsilon$.
For the algorithms presented in the next sections, this stopping criterion is generalized quite easily.

\subsection{HALSX with smoothing splines}

The computation considered above can estimate an NMF with linear features fairly efficiently.
However, in real applications, linear link functions are too restrictive.
In the following, we will estimate non-linear link functions that are Generalized Additive Models (GAM, \cite{wood_generalized_2006}).

A Generalized Additive Model is a generalization to Generalized Linear Model (GLM) which includes additive non-linear components. 
Consider $n$ observations $\mx{x}_i, y_i$, for $1 \leq i \leq n$, where $\mx{x}_i$ is the vector of features, and $y_i$ is an observation of a random variable $Y_i$.
Suppose that $Y_i = \mu_i + \epsilon_i$, where $\epsilon_i$ are independent identically distributed zero-mean Gaussian variables, and $\mu_i = \mathbf{E}(Y_i)$ has the following relationship to the features:
\begin{align*}
g(\mu_i) = \mx{x}_i^T \mx{\theta} + h_1(x_{i,1}) + h_2(x_{i,2}) + h_3(x_{i,3}, x_{i,4}) + ...
\end{align*}
where~$\mx{\theta}$ is the vector of parametric model components,~$g$ is a known, monotonic, twice-differentiable function,~$h_1, h_2$,~$h_3$,~..., are the non-linear functions to be estimated.

We note by $\mx{X}$ the matrix grouping the features of all observations.
We use penalized regression spline to fit the GAMs. 
For~$j = 1,2,3,...$, define a spline basis~$ \mx{a}^j = (a_1^j, a_2^j,...)$ in which~$h_j$, the~$j$-th component of the GAM, is to be estimated. 
Practically, we search for~$h_j$ is the~$L$-dimensional vector space
\begin{align*}
H(\mx{a}^j, L_j) = \{ \sum_{l = 1}^{L_j} \beta^j_l a^j_l | \mx{\beta}^j =  (\beta^j_1,... \beta^j_{L_j}) \in \R^{L_j} \}.
\end{align*}
Noting by~$\mx{X}^j = \{ a_l^j(\mx{x}_i) \}_{i,l}$ the design matrix, for~$h_j  = \sum_{l = 1}^{L_j} \beta^j_l a^j_l \in H(\mx{a}^j, L_j)$, an element of the functional space, we have
\begin{align*}
h_j(\mx{X}) = \mx{X}^j \mx{\beta}^j.
\end{align*}
The whole model of~$g$, can then be represented linearly:
\begin{align*}
g(\mx{\mu}) &= \mx{X} \theta + (\mx{X}^1, \mx{X}^2, ...) 
\begin{pmatrix}
\mx{\beta}^1\\ \mx{\beta}^2 \\ \colon
\end{pmatrix}\\
&= \mx{X} \mx{\theta} + \sum_{j} \mx{X}^j \mx{\beta}^j.
\end{align*}

The dimension of~$H(\mx{a}^j, L_j)$,~$L_j$, controls the the smoothness of the functions to be estimated. 
As little information is available on the degree of smoothness of the functions, we use a rather large~$L_j$, and add a penalty on the wiggliness,~$\int (h_j'')^2 d x$, as in \cite{wood_generalized_2006}.
The least squares estimator of this model is therefore
\begin{align*}
\arg\min_{\mx{\theta}, \mx{\beta}^1, \mx{\beta}^2, ... } \| g(\mx{\mu}) - \mx{X} \mx{\theta} - \sum_{j} \mx{X}^j \mx{\beta}^j \|^2 + \sum_{j} \lambda^j (\mx{\beta}^j)^T \mx{S}^j \mx{\beta}^j,
\end{align*}
where~$\lambda_j$ is the penalization parameter of the~$j$-th non-linear component, and~$\mx{S}^j$ is a positive definite matrix depending on~$\mx{X}$ and~$\mx{a}^j$.
The penalization parameter,~$\lambda^j$, is chosen by a generalized cross validation criterion.

\subsubsection*{HALSX-GAM}

At each iteration of the algorithm, for~$i = 1,...,k$, we re-estimate the link function~$f_{r,i}$ of the~$i$-th column of~$\mx{F}_r$ as a GAM.

The subproblem for~$i$ is the following
\begin{equation}
  \begin{aligned}
   \arg \min_{\mx{\theta}, \mx{\beta}^1, \mx{\beta}^2, ... } & \| \mx{R}_i^t - (\mx{X}_r \mx{\theta} + \sum_{j = 1} \mx{X}^j \mx{\beta}^j) (\mx{f}_{c, i}^t)^T \|_F^2 + \\
   & \sum_{j} \lambda^j (\mx{\beta}^j)^T \mx{S}^j \mx{\beta}^j.
  \end{aligned}
\end{equation}
With fixed penalization parameters~$\lambda_1, \lambda_2,$ ..., the optimization above can be solved by
\begin{align*} 
&
\begin{pmatrix}
\mx{\theta}\\
\mx{\beta^1}\\
\mx{\beta^2}\\
\colon
\end{pmatrix}^{t + 1}
 = 
 \frac{1}{\|f_{c,i}^t\|^2} \times \\
&
\begin{pmatrix}
& \mx{X}_r^T\mx{X}_r & \mx{X}_r^T\mx{X}^1 & \mx{X}_r^T\mx{X}^2 & \cdots \\
& (\mx{X}^1)^T\mx{X}_r & (\mx{X}^1)^T\mx{X}^1 + \frac{\lambda^1}{\|f_{c,i}^t\|^2} \mx{S}^j &  (\mx{X}^1)^T\mx{X}^2 & \cdots  \\
& \colon & \colon & \ddots & \cdots
\end{pmatrix}^{-1} \times \\
&
\begin{pmatrix}
\mx{X}_r^T\\
(\mx{X}^1)^T\\
(\mx{X}^2)^T\\
\colon
\end{pmatrix}
\mx{R}_i^t\mx{f}_{c,i}^t.
\end{align*}

In practice, we use the GAM estimation routines implemented in the \textit{R} package \textit{mgcv} \cite{wood_generalized_2006} to choose the penalization parameter and estimate the model at the same time.

\subsection{HALSX with other regression models}

We can replicate the strategy above to work with other regression models. 
As mentioned before, the convergence to critical point is guaranteed, as long as the regression model estimation can be re-parameterized to verify the conditions of Proposition \ref{prop:jacobian}.
Using this strategy, many off-the-shelf algorithms for regression model training can be plugged in.
In the experiments described in the next section, we use the predictive model API provided in the \textit{R} package \textit{caret} \cite{kuhn_building_2008}.

\subsubsection*{Meta-parameters in regression models}
As in HALSX-GAM, we build the estimation of meta-parameters using cross validation as a part of the link function estimation step, can treat them indifferently as regular parameters.

\subsection{An HALS-like algorithm for (\ref{eq:general_optimization_problem2})}\label{sec:HALSX2}

Before detailing Algorithm \ref{algo:halsx2} which aims to solve (\ref{eq:general_optimization_problem2}), we will first develop the elemental HALS iteration in the context of (\ref{eq:general_optimization_problem2}) where no supplemental information is supplied for the factorization model, namely~$\mx{X}_r = \mx{I}_{n_1}, \mx{X}_c = \mx{I}_{n_2}$. 
Indeed, when updating one column of~$\mx{F}_r$, the sub-problem becomes:
how to solve~$\arg \min_{\mx{f}} \| \mx{b} - \operator{A}(\mx{f}(\mx{f}_{c})^T )\|_2^2$?

We will use the fact that for all~$\mx{M} \in \R^{n_1 \times n_2}$,
\begin{align*}
\operator{A}(\mx{M}) = (\langle \mx{A}_i, \mx{M} \rangle)_{1 \leq i \leq N},
\end{align*}
and for all~$\mx{b} \in \R^N$,~$\operator{A}^\ast$, the transpose of~$\operator{A}$ is defined by
\begin{align*}
 \operator{A}^\ast(\mx{b}) = \sum_{i = 1}^N b_i \mx{A}_i.
\end{align*} 
Since 
\begin{align*}
\frac{\partial}{\partial \mx{f}} \| \mx{b} - \operator{A}( \mx{f} (\mx{f}_{c})^T )\|_2^2 
& = \operator{A}^\ast [\operator{A}(\mx{f} (\mx{f}_{c})^T) - \mx{b}] \mx{f}_c,
\end{align*}
the first order optimality condition~$\frac{\partial}{\partial \mx{f}} \| \mx{b} - \operator{A}( \mx{f} (\mx{f}_{c})^T )\|_2^2  0$ is therefore equivalent to 
\begin{align*}
\operator{A}^\ast [\operator{A}(\mx{f} (\mx{f}_{c})^T)] \mx{f}_c = \operator{A}^\ast [\mx{b}] \mx{f}_c.
\end{align*}
The left-hand side of the equation can be written as 
\begin{align*}
\operator{A}^\ast [\operator{A}(\mx{f} (\mx{f}_{c})^T)] \mx{f}_c 
& = (\sum_{i = 1}^N \langle \mx{A}_i, \mx{f} (\mx{f}_{c})^T) \rangle \mx{A}_i ) \mx{f}_c \\
& = \sum_{i = 1}^N \text{Tr} ( \mx{f}  (\mx{A}_i \mx{f}_{c})^T)  (\mx{A}_i  \mx{f}_c) \\
& = \sum_{i = 1}^N  (\mx{A}_i  \mx{f}_c) \text{Tr} (  (\mx{A}_i \mx{f}_{c})^T \mx{f} ) \\
& = \sum_{i = 1}^N  (\mx{A}_i  \mx{f}_c) (\mx{A}_i \mx{f}_{c})^T   \mx{f},
\end{align*}
which leads to the following symmetric~$n_1$-by-$n_1$ system on~$\mx{f}$:
\begin{align*}
(\sum_{i = 1}^N  (\mx{A}_i  \mx{f}_c) (\mx{A}_i \mx{f}_{c})^T)   \mx{f} = \sum_{i = 1}^N b_i \mx{A}_i \mx{f}_c,
\end{align*}
or 
\begin{align*}
\mx{f} = (\sum_{i = 1}^N  (\mx{A}_i  \mx{f}_c) (\mx{A}_i \mx{f}_{c})^T)^{-1} \sum_{i = 1}^N b_i \mx{A}_i \mx{f}_c
\end{align*}

This computation generalizes to linear exogenous variables, with the optimality condition:
\begin{align*}
\mx{\beta} = (\sum_{i = 1}^N  (\mx{X}\mx{A}_i  \mx{f}_c) (\mx{X}\mx{A}_i \mx{f}_{c})^T)^{-1} \sum_{i = 1}^N b_i \mx{A}_i \mx{f}_c.
\end{align*}

When the matrices to be inversed in the these equations are not invertible, we will use the generalized inverse instead.

Using these elementary steps, we propose Algorithm \ref{algo:halsx2} to solve Problem (\ref{eq:general_optimization_problem2}).
Compared to Algorithm \ref{algo:halsx}, Algorithm~\ref{algo:halsx2}
\begin{itemize}
\item has one less block (the slack variable~$\mx{V}$ is not present);
\item checks the deviation with data more frequently;
\item each subproblem is more costly because of the presence of~$\operator{A}$ in the subproblem. As we will see in the detailed development of the computation, when~$N$, the sample size (the dimension of image of~$\operator{A}$) is large, each update involves rather costly computations.
\end{itemize}

\begin{algorithm}
\caption{Hierarchical Alternating Least Squares with eXogeneous variables for NMF (HALSX2)}
\label{algo:halsx2}
\begin{algorithmic}
\REQUIRE Measurement operator~$\operator{A}$, measurements~$\mx{b}$, rank~$1\leq k \leq \min\{n_1, n_2\}$, features~$\mx{X}_r$ and~$\mx{X}_c$, functional spaces~$F_r$ and~$F_c$ in which to search the link functions.\;
\STATE Initialize~$\mx{F}_r^0, \mx{F}_c^0 \geq 0,  t = 0$\;
\WHILE{Stopping criterion is not satisfied}
\STATE~$\mx{R}^t = \mx{b} - \operator{A}(\mx{F}_r^t (\mx{F}_c^t)^T)$
\FOR{$i = 1,2,...,k$}
\STATE~$\mx{R}^t = \mx{R}^t + \operator{A}(\mx{f}_{r,i}^t (\mx{f}_{c,i}^t)^T)~$
\STATE~$f_{r,i}^{t+1} = \arg \min_{f \in F_r} \| \mx{R}^t - \operator{A}(f(\mx{X}_r) (\mx{f}_{c,i}^t)^T )\|_2^2$\;
\STATE~$\mx{f}_{r,i}^{t+1} = \max(0, f_{r,i}^{t+1}(\mx{X}_r))$
\STATE~$\mx{R}^t = \mx{R}^t - \operator{A}(\mx{f}_{r,i}^{t+1} (\mx{f}_{c,i}^t)^T)~$
\ENDFOR
\FOR{$i = 1,2,...,k$}
\STATE~$\mx{R}^t = \mx{R}^t + \operator{A}(\mx{f}_{r,i}^{t+1} (\mx{f}_{c,i}^t)^T)~$
\STATE~$f_{c,i}^{t+1} = \arg \min_{f \in F_c} \| \mx{R}^t - \operator{A}(\mx{f}_{r,i}^{t+1} f(\mx{X}_c)^T) \|_2^2$\;
\STATE~$\mx{f}_{c,i}^{t+1} = \max(0, f_{c,i}^{t+1}(\mx{X}_c))$
\STATE~$\mx{R}^t = \mx{R}^t - \operator{A}(\mx{f}_{r,i}^{t+1} (\mx{f}_{c,i}^{t+1})^T)~$
\ENDFOR
\STATE ~$t = t+1$\;
\ENDWHILE
\RETURN~$\mx{F}_r^t \in \mathbb{R}_+^{n_1 \times k}, f_{r,1}^t,..., f_{r,k}^t \in F_r$,\\
~$\mx{F}_c^t \in \mathbb{R}_+^{n_2\times k}, f_{c,1}^t,..., f_{c,k}^t \in F_c.$
\end{algorithmic}
\end{algorithm}

\subsubsection*{Complexity}
At each sub-iteration Algorithm \ref{algo:halsx2}, we need to calculate~$N$~$n_1$-by-$n_1$ or~$n_2$-by-$n_2$ matrices ($(\mx{A}_i  \mx{f}_c) (\mx{A}_i \mx{f}_{c})^T$), then inverse the sum of the these~$N$ matrices.
While the computation of the sum is map-reducible, on a single-threaded machine, this can be computationally expensive when~$N$ is large. 
Each iteration of Algorithm \ref{algo:halsx2} has a multiplicative complexity of~$O(kN (n_1^2 + n_2^2))$, while each iteration of Algorithm \ref{algo:halsx} has a complexity of~$O(k (n_1^2 + n_2^2) + N n_1 n_2)$ with general linear measurement operator.
This difference in complexity can be very important when~$N$ or~$k$ is large.
% With linear measurement operators with efficient projection methods, the complexity could become~$O(N) + O(k n_1 n_2)$.
% This computation cost will be discussed more in detail in the experiements described in the next section.

\section{Experiments\label{sec:experiments}}

We use one synthetic dataset and three real datasets to evaluate the proposed methods. 
\begin{itemize}
\item \textbf{Synthetic data}: a rank-20 150-by-180 nonnegative matrix simulated following the generative model (Section~\ref{sec:general model}), with~$\mx{X}_r\in \R^{150\times3}$ and~$\mx{X}_c \in \R^{180\times4}$ matrices with independent Gaussian entries,~$f_r: \R^3 \rightarrow \R^ {20}$ ($f_c : \R^4 \rightarrow \R^ {20}$) is a function formed with a dimension-33 (44 for~$f_c$) spline basis with random weights, truncated at 0 ($T = 150, N = 180$). 
\item \textbf{French electricity consumption} (proprietary dataset): daily consumption of 473 medium-voltage feeders gathering each around 1,500 consumers based near Lyon in France from 2010 to 2012. 
The first two years are used as training data ($T = 1096, N = 473$).
\item \textbf{Portuguese electricity consumption} \cite{UCI_ML} daily consumption of 370 Portuguese clients from 2010 to 2014 ($T = 1461, N = 369$).
\item \textbf{MovieLens 100k} \cite{_movielens_2015} an anonymized public dataset with 100,000 movie scores for 1682 movies from 943 users ($T = 943, N = 1682$).
Note that the data matrix is not complete.
Error rates are calculated on the vector of available scores.
\end{itemize}

The following matrix recovery/completion methods are compared:
\begin{itemize}
\item \textbf{interpolation} For temporal aggregates only: to recover the target matrix, temporal aggregates are distributed equally over the covered periods.
\item \textbf{HALS, NeNMF, softImpute} Matrix recovery/completion methods without side information.
\end{itemize}
The following regression methods are compared:
\begin{itemize}
\item \textbf{individual\_gam} Estimating separate GAMs on each individual or period, on the matrix obtained from \textbf{interpolation} or on the whole data matrix when it is available.
\item \textbf{factor\_gam} Estimating GAMs on the factors obtained by \textbf{HALS} or \textbf{NeNMF}.
\item \textbf{rrr} \cite{addy_reduced-rank_????} Applying reduced-rank regression on the matrix obtained from \textbf{interpolation} or on the whole data matrix when it is available.
\item \textbf{grmf} \cite{rao_collaborative_2015} A matrix completion algorithm using graph-based side information to enhance collaborative filtering performance.
\item \textbf{trmf} \cite{yu_high-dimensional_2015} A matrix completion algorithm tailored to time series, by adding three penalization terms to the matrix factorization quadratic error. When only temporal aggregate measurements are available, we apply this method on the matrix obtained from \textbf{interpolation}.
\item \textbf{HALSX\_model} Algorithm \ref{algo:halsx}.
\end{itemize}

For all matrix factorization methods (HALS, NeNMF, rrr, factor\_gam, grmf, trmf, HALSX\_model), we use the method with several ranks, then choose the best rank 
($k \in \{2,3,...,20\}$ for synthetic and Portuguese data,~$k \in \{2,3,...,10\}$ for French and MovieLens data).
For trmf, we do a grid search on the three penalization parameters, and choose the best combination.
For HALSX\_model, we use four different regression models: linear model, GAM, support vector regression with linear kernel, and Gaussian process regression with radial basis kernel (\textbf{lm}, \textbf{gam}, \textbf{svmLinear}, \textbf{gaussprRadial}).
We use the implementation of \textbf{lm} in standard \textit{R}, and \textbf{svmLinear} and \textbf{gaussprRadial} of the \textit{R} package \textit{kernlab}\cite{karatzoglou_kernlab-s4_2004}, through the \textit{caret} API. 
As of \textbf{gam}, we use the \textit{mgcv} implementation directly.

For each data matrix, we apply a linear measurement operator on an upper-left submatrix to obtain the measures (of dimension 100-by-130 for synthetic data, 730-by-270 for French data, 731-by-369 for Portuguese data, and 666-by-1189 for MovieLens data). 
On these measurements of the submatrix, we use each method to estimate a matrix factorization model, with or without side information.
We then report the matrix recovery error on this sampled submatrix for all methods, and the prediction errors on the rest of the data matrix for the methods can produce predictions for new columns and/or rows. 
We distinguish row prediction error, column prediction error, and row-column prediction error, as the error calculated on the lower-left, upper-right and lower-right submatrix.

We report the relative root-mean-squared error as error metric in this section: 
\begin{itemize}
\item for electricity datasets, $\text{RRMSE}(\mx{V}, \mx{V}^\ast) = \frac{\| \mx{V} - \mx{V}^\ast\|_F}{\|\mx{V}^\ast\|_F}~$,
\item for MovieLens, we calculate the $l_2$-norm version on the vector of all available movie scores.
\end{itemize}
Other error metrics do not seem to be qualitatively different in our experiments.

\subsection{Execution time and precision between HALS and HALS2}

To show-case the complexity difference discussed in Section~\ref{sec:HALSX2}, we run HALS (Algorithm~\ref{algo:halsx}) and HALS2 (Algorithm~\ref{algo:halsx2}) on a random temporal aggregate mask, with no side information.
In HALS, we treated the mask as a general mask, without using the acceleration specific to temporal aggregate masks.
The per iteration execution time is shown in Figure \ref{fig:time per iteration hals}. 
The difference in complexity discussed in Section \ref{sec:HALSX2} is fairly clear here. 
In HALS2, the execution time per iteration increases both with the rank and the number of samples in the data, at least for sampling rate from 14.3\% on. 
For low sampling rates, HALS2 often diverges.

As discussed in Section \ref{sec:HALSX2}, although the execution time per iteration is greater, HALS2 could be more efficient if it does much less iterations than HALS.
In Figure \ref{fig:time hals}, we can see that this is not the case: 
for problems with high sampling rates, HALS2 indeed does less iterations to converge. 
However, the total execution time is still larger than HALS.
For lower sampling rates, HALS2 has troubles converging, and only stops when the maximal execution time allowed (300 seconds) is reached.
This is also confirmed in Figure \ref{fig:reconstitution hals}, where HALS2 has much worse recovery error than HALS.

\begin{figure}
\begin{center}
\includegraphics[width=0.5\textwidth]{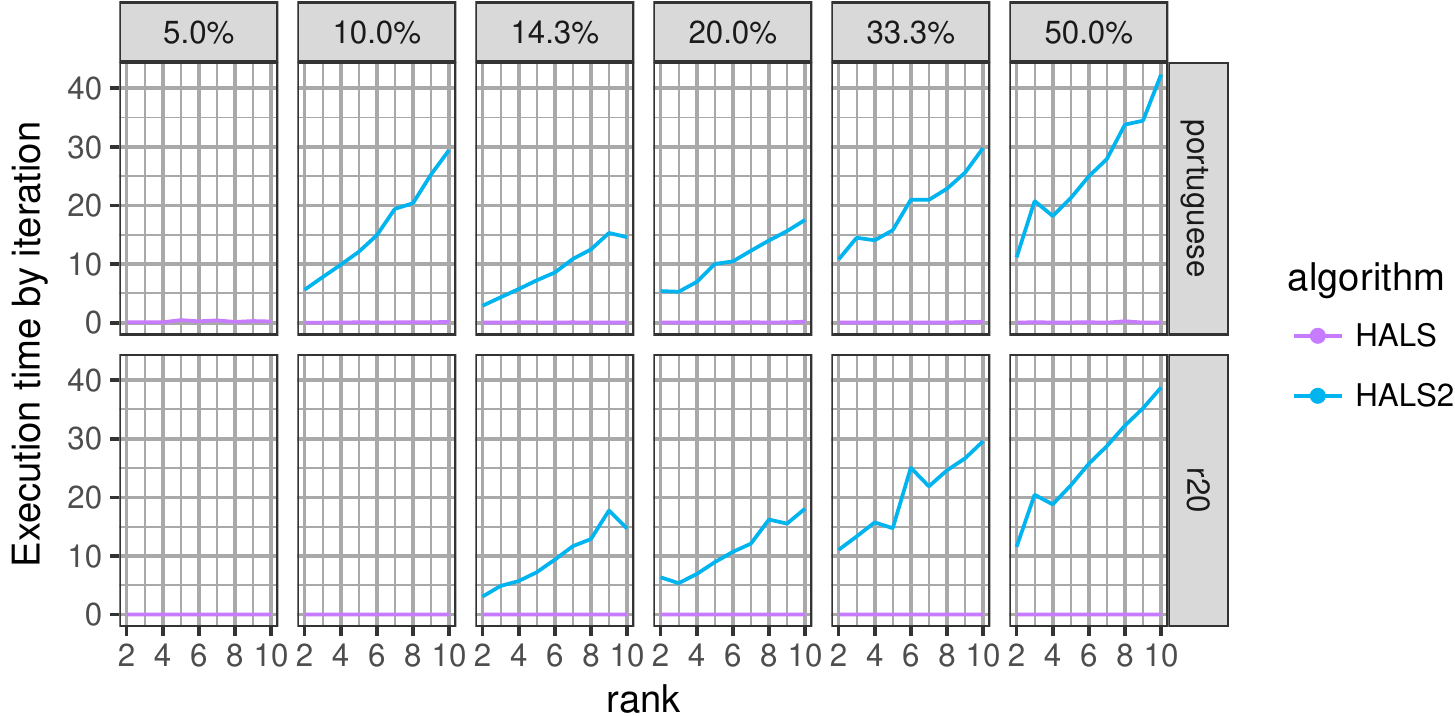}
\caption{Execution time per iteration of Algorithm \ref{algo:halsx} and Algorithm \ref{algo:halsx2} without exogenous variables}
\label{fig:time per iteration hals}
\end{center}
\end{figure}

\begin{figure}
\begin{center}
\includegraphics[width=0.5\textwidth]{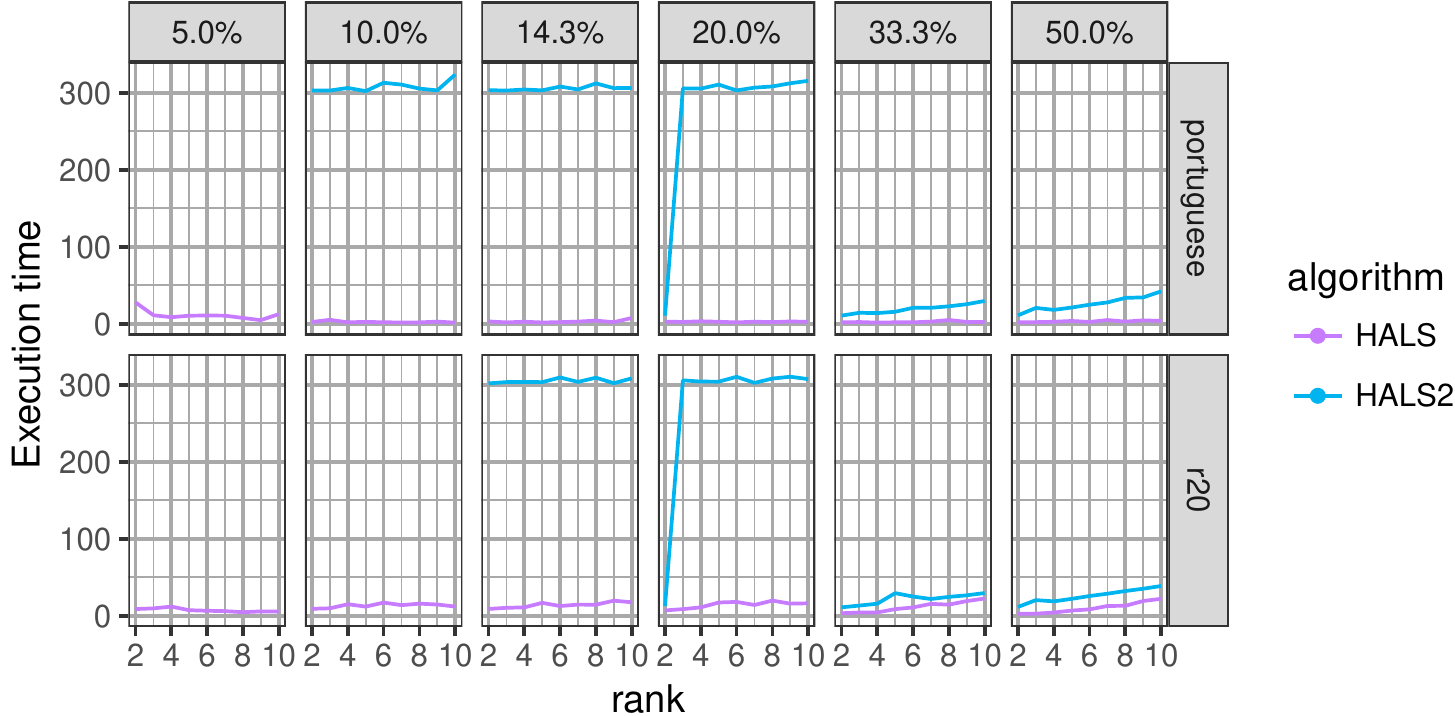}
\caption{Execution time of Algorithm \ref{algo:halsx} and Algorithm \ref{algo:halsx2} without exogenous variables}
\label{fig:time hals}
\end{center}
\end{figure}
\begin{figure}
\begin{center}
\includegraphics[width=0.5\textwidth]{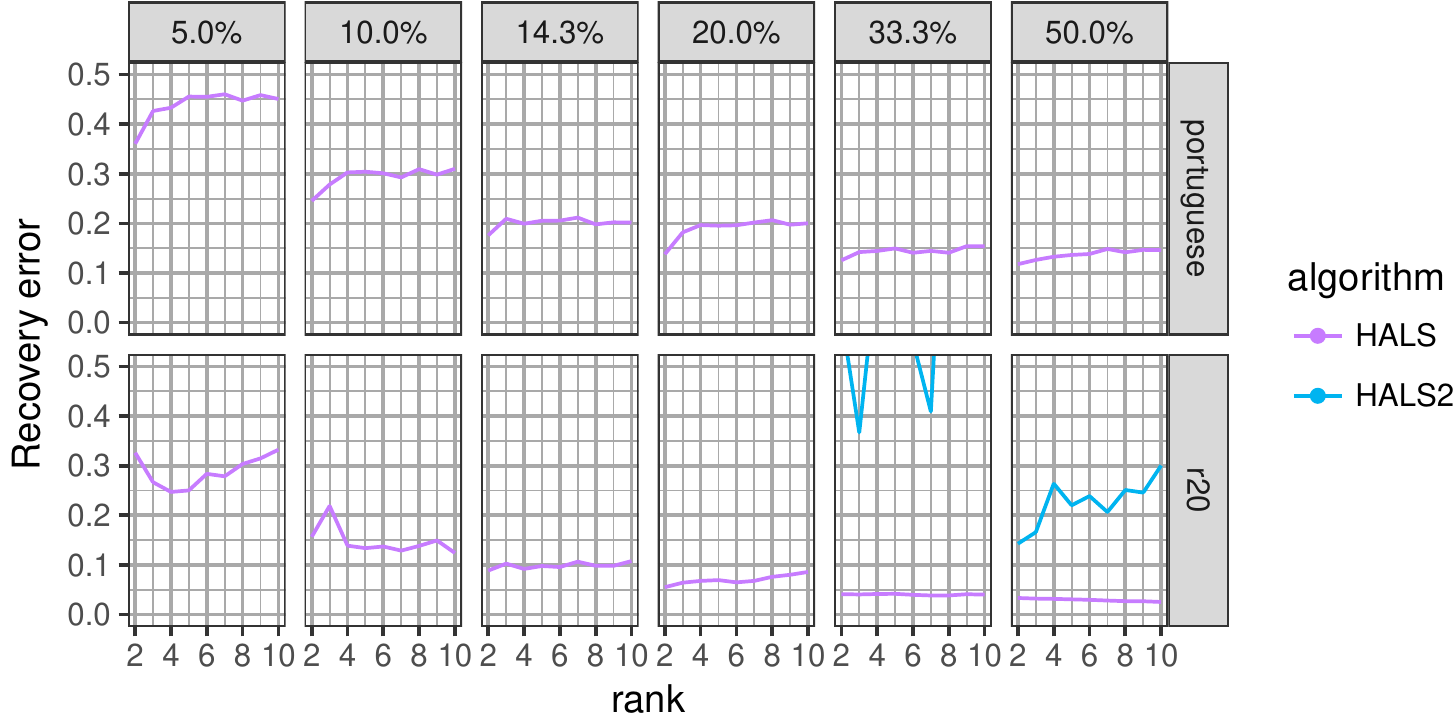}
\caption{Reconstitution precision of Algorithm \ref{algo:halsx} and Algorithm \ref{algo:halsx2} without exogenous variables}
\label{fig:reconstitution hals}
\end{center}
\end{figure}

\subsection{Performance on temporal aggregate measurements}

On the synthetic and the two real electricity consumption datasets, we use Algorithm \ref{algo:halsx} to perform matrix recovery and prediction on temporal aggregate measurements.
Every method except for \textbf{grmf} is used in this setting.

We use two types of temporal aggregate measurements: periodic and random.
In periodic measurements, each scalar measure covers a fixed number of periods of one individual.
In random measurements, the number of periods covered by a measure is random (see \cite{mei_nonnegative_2017} for more details).
In electricity consumption data, periodic measurements are closer to the actual meter reading schedules of utility companies, while matrix recovery with random measurements is an easier problem.
For both sampling types, we sample 10\%, 20\%,~..., 50\% of the data to show-case the matrix recovery performance of the proposed method.

For the synthetic data, we use the true row and column features used to produce the simulations.
For the French electricity data, the row features are variables known to have an influence on electricity consumption: the temperature, the day type (weekday, weekend, or holiday), the position of the year.
The column features are the percentage of residential, professional, or industrial usages in the group of users for each column.
For the Portuguese electricity data, as no individual features are available, we only use the same row features as for the French dataset (temperature, day type, position of the year).

Figure \ref{fig:time series recovery} shows the matrix recovery error.
For most of the scenario, HALSX\_models (red lines with symbols) are comparable or better than the other methods without side information.
The only case where an HALSX\_model is a little worse is when compared with HALS and NeNMF (two NMF methods) in synthetic data with random measurements, which is the least close to the real application.
The \textbf{softImpute} \cite{mazumder_spectral_2010} method is not well adapted to temporal aggregate measurements, and has much higher error (higher than the maximal value in these graphics.

\begin{figure}
\begin{center}
\includegraphics[width=0.5\textwidth]{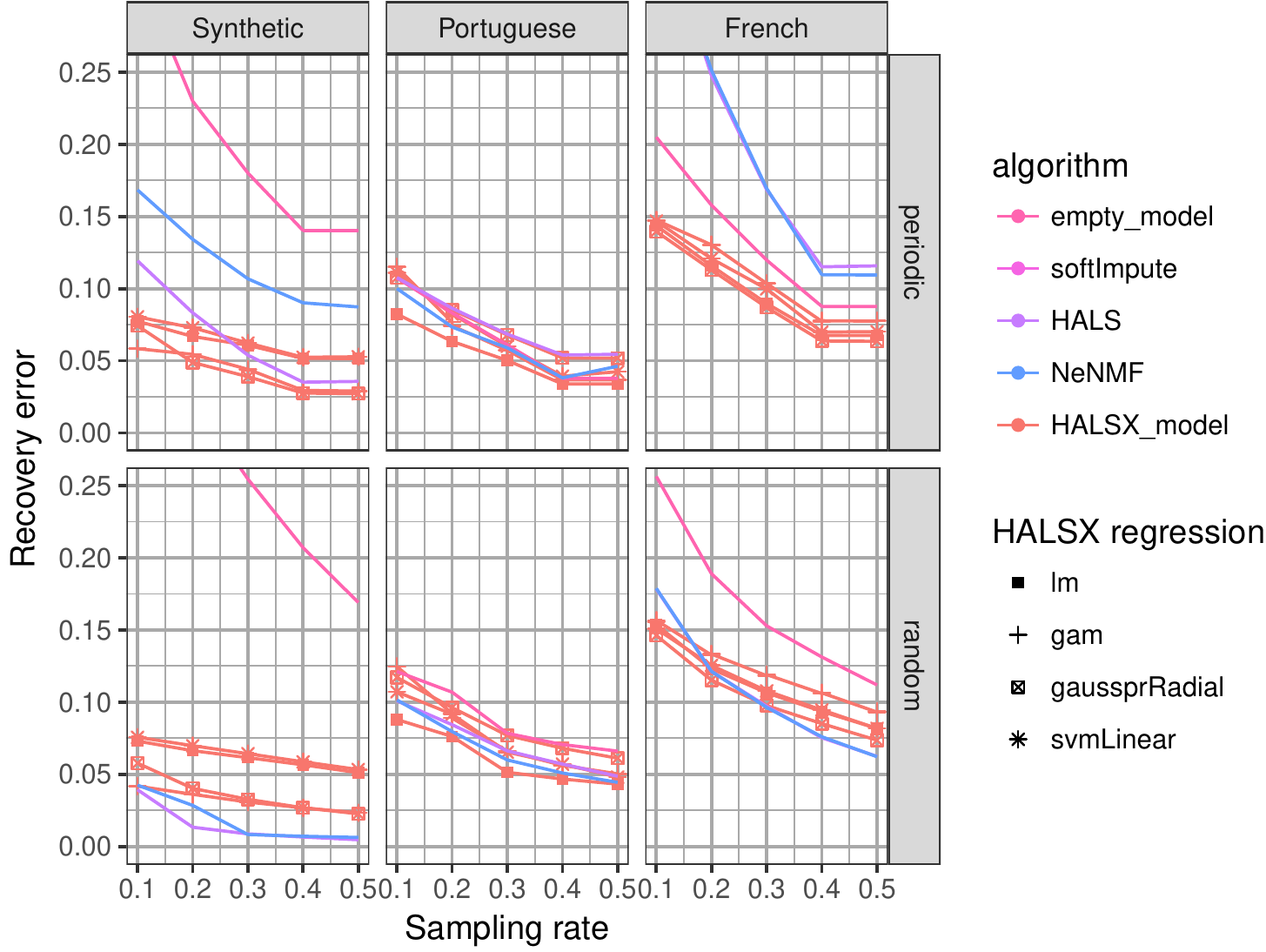}
\caption{Recovery performance on synthetic and real electricity data}
\label{fig:time series recovery}
\end{center}
\end{figure}

Figures \ref{fig:synthetic}, \ref{fig:french}, and \ref{fig:portuguese} show the prediction error on the three datasets.
We can see that \textbf{trmf} \cite{yu_high-dimensional_2015} and \textbf{rrr} \cite{addy_reduced-rank_????} not very adapted to temporal aggregate measurements.
When they are applicable (\textbf{trmf} is only applicable to row prediction, \textbf{rrr} only applicable to row or column predictions, not RC predictions), they have much worse performance than the other methods, except in Synthetic data with complete observations.

In most cases, HALSX\_models are comparable to or better than factor\_gam and individual\_gam, which shows that using side information while estimating the factorization model produces factors more adapted for prediction.
It is interesting to note that in some cases, using HALSX\_models with incomplete data (sampling rate less than 100\%) is actually better than using individual\_gam with complete data.
This means that compared to traditional regression methods, the proposed method achieves better prediction models using much less data, by exploiting the low-rank structure of the problem.

Moreover, the performance HALSX\_models is the least sensitive to sampling rates: it is mostly constant from 30\% of data. 
This shows that using side information is supplementary to observing more data.

\begin{figure}  
\begin{center}
\includegraphics[width=0.5\textwidth]{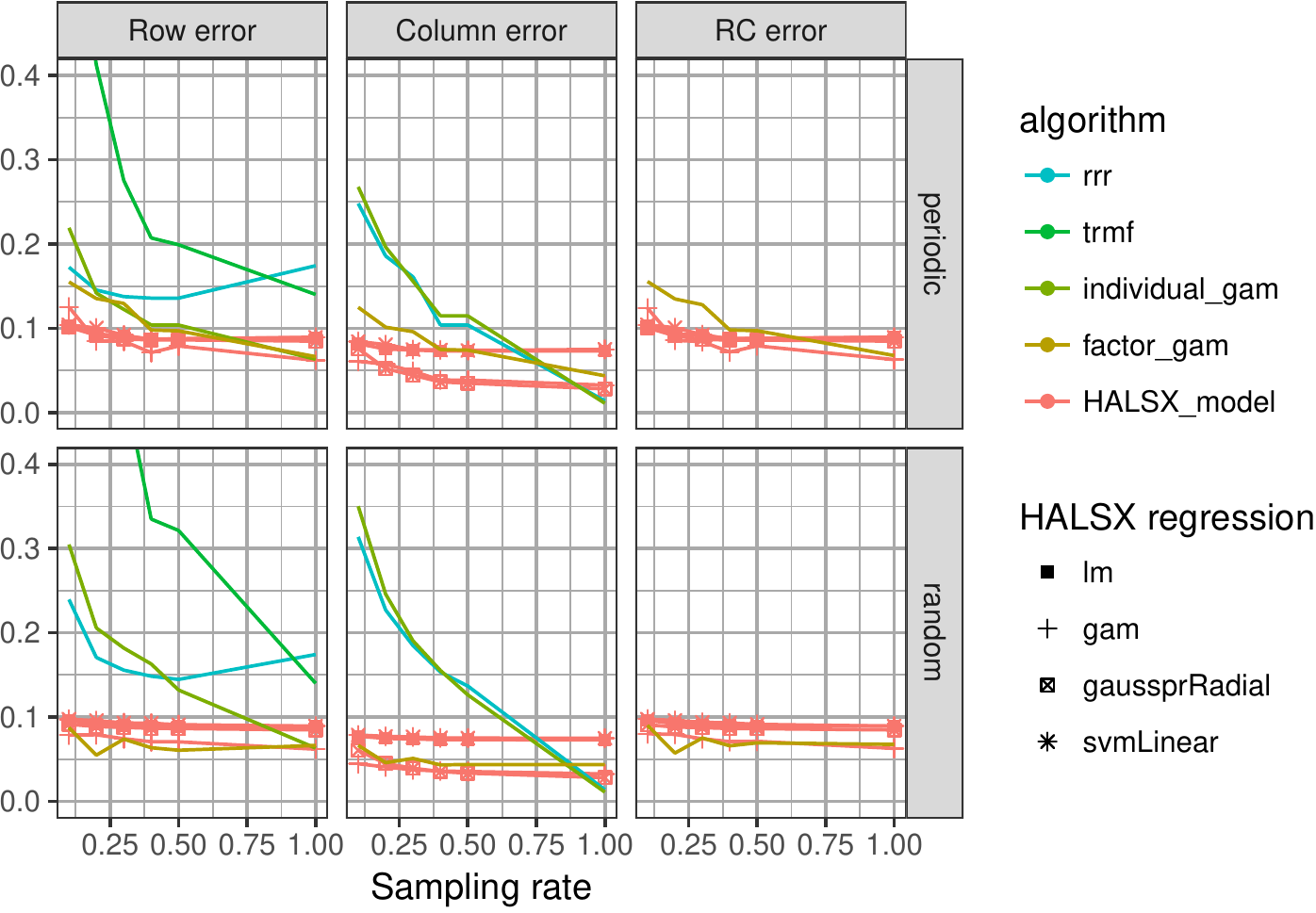}
\caption{Prediction performance on synthetic data}
\label{fig:synthetic}
\end{center}
\end{figure}
\begin{figure}  
\begin{center}
\includegraphics[width=0.5\textwidth]{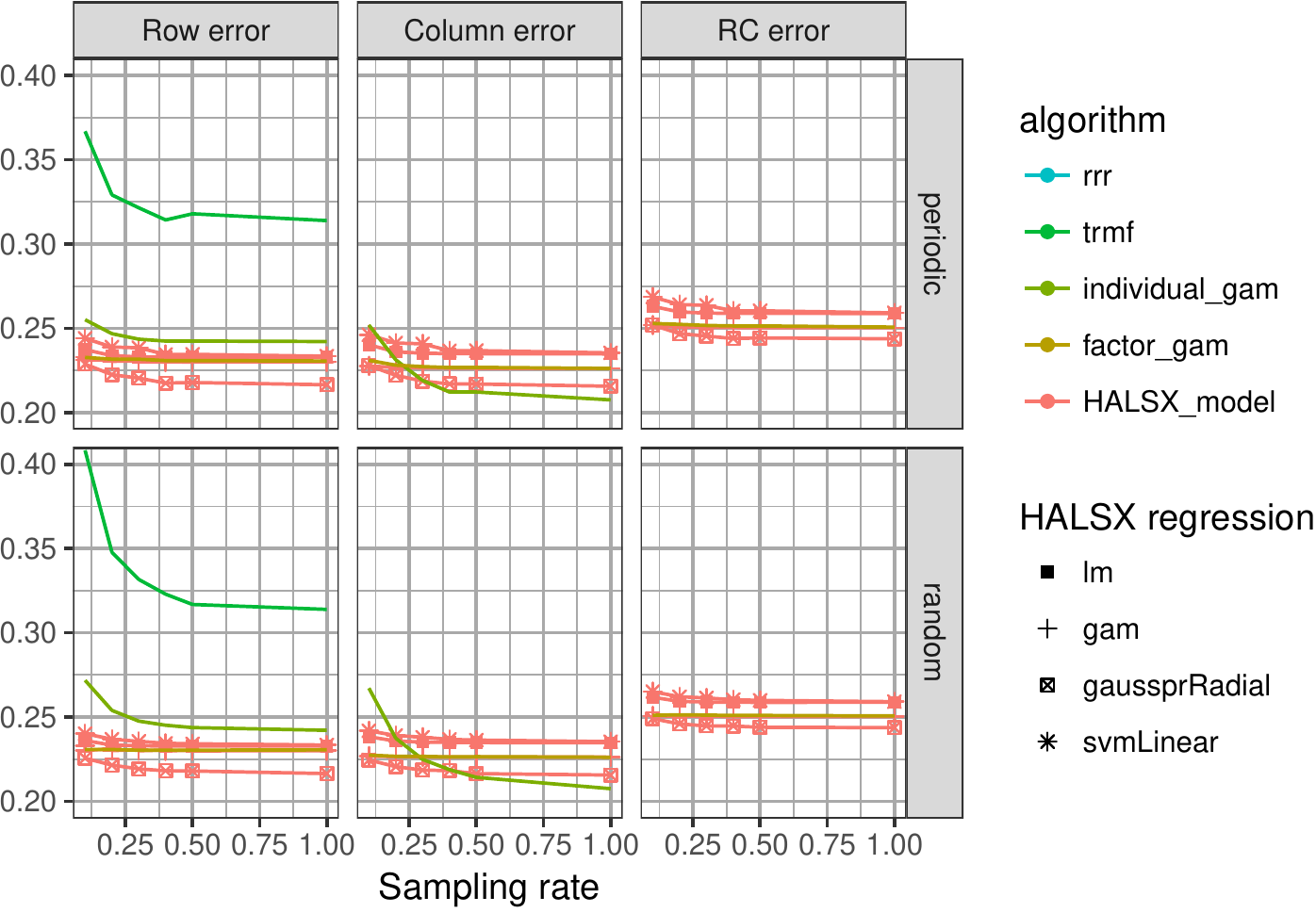}
\caption{Prediction performance on real French electricity data}
\label{fig:french}
\end{center}
\end{figure}

\begin{figure}  
\begin{center}
\includegraphics[width=0.5\textwidth]{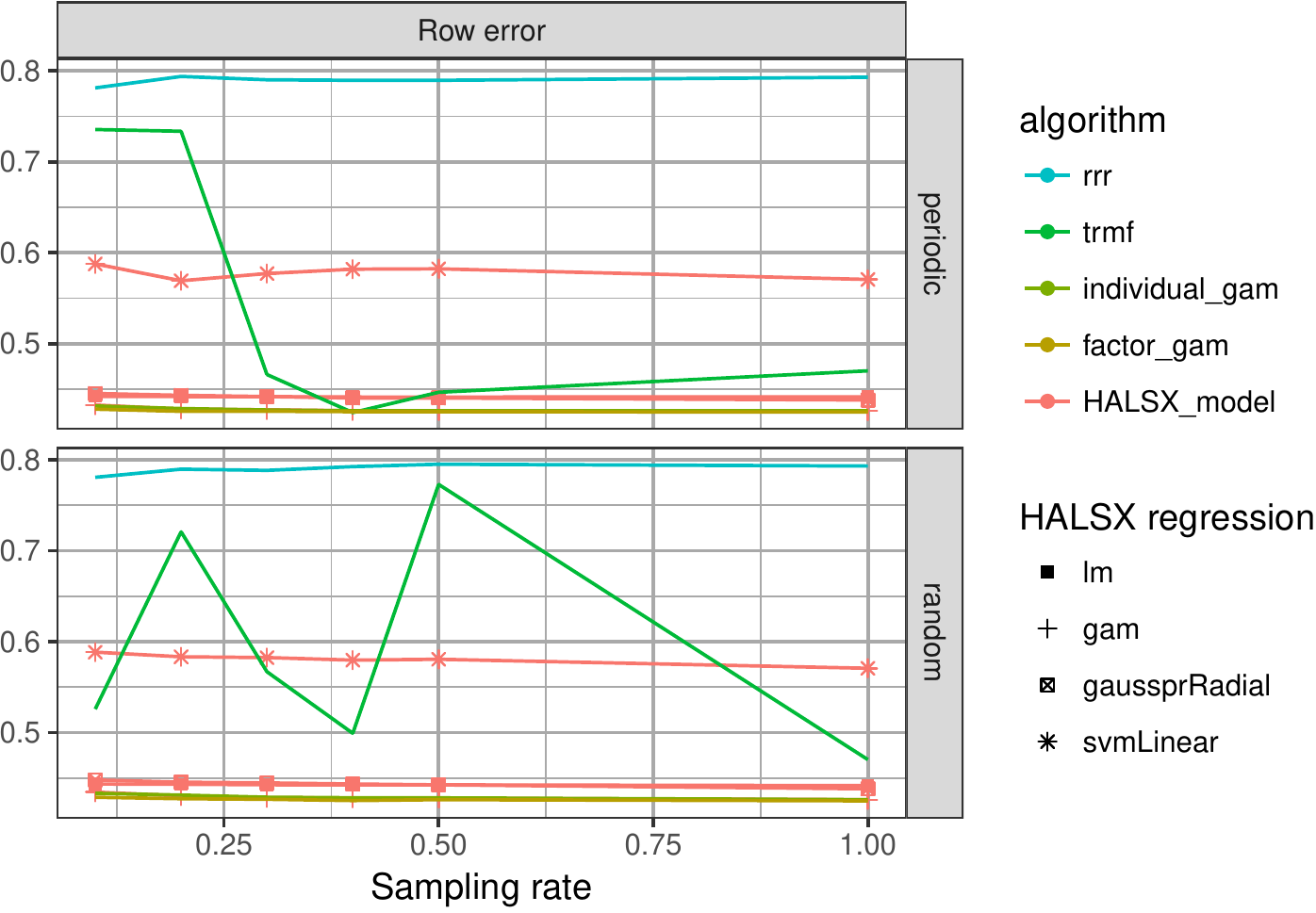}
\caption{Prediction performance on real Portuguese electricity data}
\label{fig:portuguese}
\end{center}
\end{figure}

\subsection{Performance on matrix completion mask}

On the MovieLens dataset, we use Algorithm \ref{algo:halsx} to perform matrix recovery and prediction with uniformly sampled matrix entries.
The sampling rates are~$10\%, 20\%, 30\%, 40\%, 50\%, 90\%$.

As is the case for temporal aggregate measurements, we use samples from the upper-left submatrix to estimate the model, evaluate matrix recovery on that submatrix, and evaluate row and/or column prediction errors.

Every method except for trmf is used in this setting.
As side information, we use the gender, the age, and the occupation of the users and the genre (a dimension-19 binary variable) of the movies.
For \textbf{grmf} \cite{rao_collaborative_2015}, we produce a graph where each individual is connected to its ten nearest neighbors with euclidean distance with the features.
As the parameters estimated in \textbf{grmf} is per user/movie, it can not be used to predict new individuals, even though it uses side information.

Figure \ref{fig:movielens complete} shows the recovery error on MovieLens data. 
We can see that in the low sampling rate case (10\%), HALSX\_model with \textbf{lm} works the best. 
In higher sampling rate cases, the NMF methods without side information (HALS and NeNMF) work the best, with HALSX\_model with \textbf{lm}, \textbf{gam} or \textbf{svmLinear} are second to best.
HALSX\_model with \textbf{gaussprRadial} does not work very in this case, as is the case with the other comparison methods.

\begin{figure}  
\begin{center}
\includegraphics[width=0.33\textwidth]{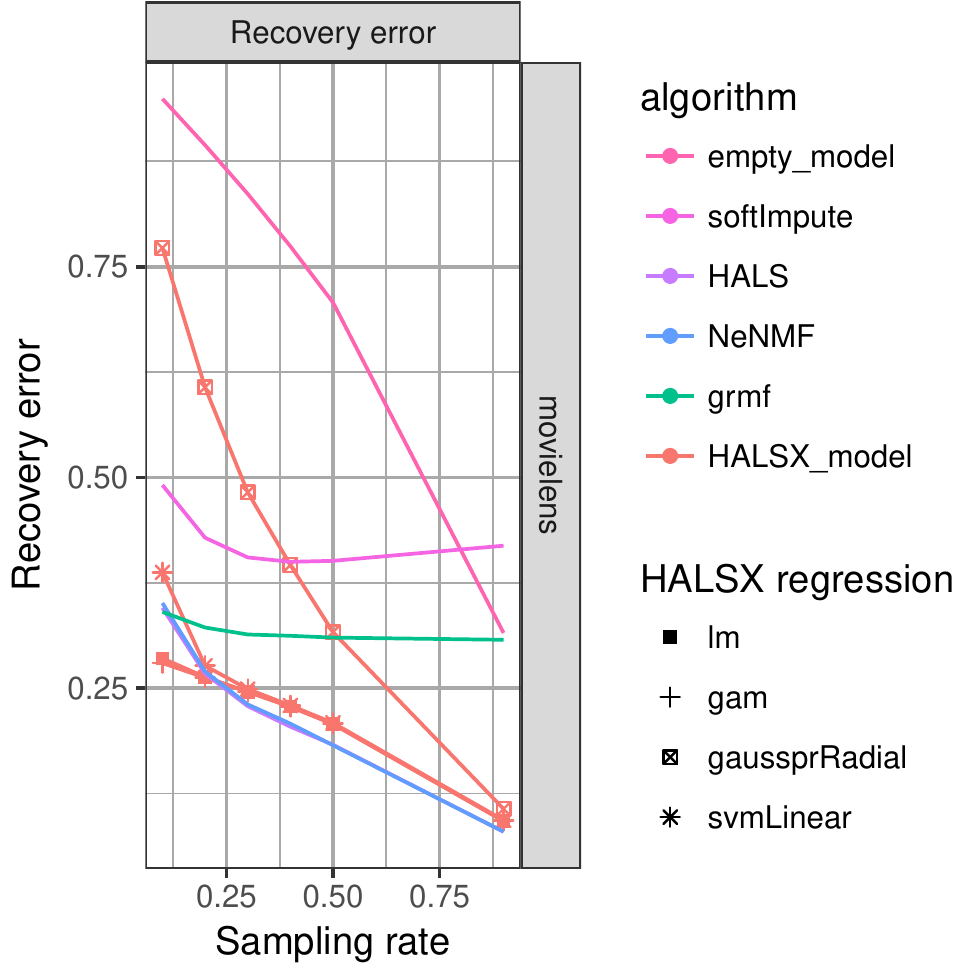}
\caption{Matrix completion performance on MovieLens 100K data}
\label{fig:movielens complete}
\end{center}
\end{figure}

Figure \ref{fig:movielens complete} shows the prediction error on MovieLens data for new users and/or new movies. 
The order of the variants of the HALSX\_models is conserved: \textbf{gam} and \textbf{lm} are the best, \textbf{svmLinear} is not very good for 10\%, but better with higher sampling rates, and \textbf{gaussprRadial} does not work well in this problem.
Otherwise, the factor\_gam method is slightly better for column predictions (new movies) in higher sampling rate cases, but worse in other cases. 
Both individual\_gam and rrr are much worse than the proposed methods.

\begin{figure}  
\begin{center}
\includegraphics[width=0.5\textwidth]{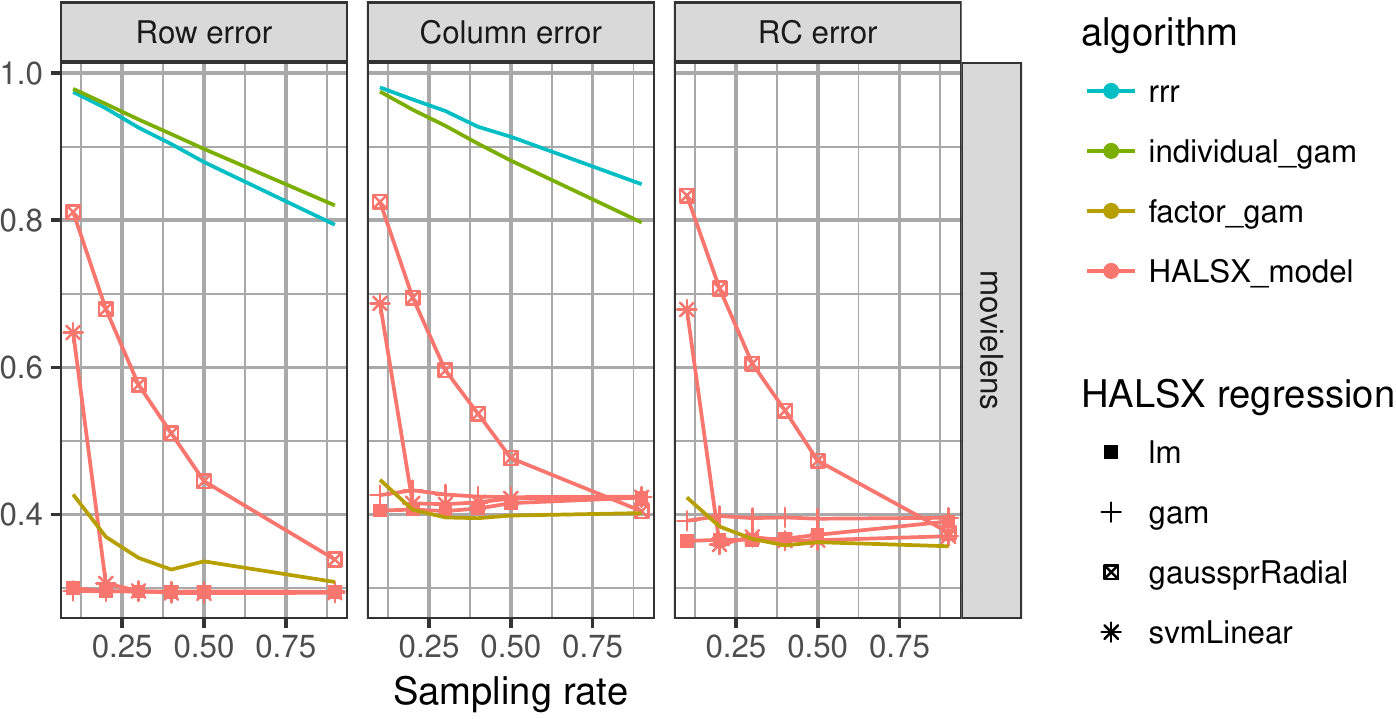}
\caption{Matrix completion performance for new rows and new columns on MovieLens 100K data}
\label{fig:movielens predict}
\end{center}
\end{figure}

\section{Conclusion}
We established a general approach for including side information on the columns and row in nonnegative matrix factorization methods, with general linear measurements.
On the theoretical front, we established a sufficient condition on the features for the factorization to be unique.
We deduced a general algorithm to solve the problem, and showed that the algorithm converges to a critical point in rather general conditions.
The proposed algorithm is compared in synthetic and real datasets in various sampling rates, and showed comparable or better empirical performance than the referenced methods.

% % use section* for acknowledgment
% \ifCLASSOPTIONcompsoc
%   % The Computer Society usually uses the plural form
%   \section*{Acknowledgments}
% \else
%   % regular IEEE prefers the singular form
% \fi

The authors would like to thank Enedis for their help with the proprietary datasets.

\bibliographystyle{IEEEtran}
\bibliography{all,mypaper}

\begin{thebibliography}{10}
\providecommand{\url}[1]{#1}
\csname url@samestyle\endcsname
\providecommand{\newblock}{\relax}
\providecommand{\bibinfo}[2]{#2}
\providecommand{\BIBentrySTDinterwordspacing}{\spaceskip=0pt\relax}
\providecommand{\BIBentryALTinterwordstretchfactor}{4}
\providecommand{\BIBentryALTinterwordspacing}{\spaceskip=\fontdimen2\font plus
\BIBentryALTinterwordstretchfactor\fontdimen3\font minus
  \fontdimen4\font\relax}
\providecommand{\BIBforeignlanguage}[2]{{%
\expandafter\ifx\csname l@#1\endcsname\relax
\typeout{** WARNING: IEEEtran.bst: No hyphenation pattern has been}%
\typeout{** loaded for the language `#1'. Using the pattern for}%
\typeout{** the default language instead.}%
\else
\language=\csname l@#1\endcsname
\fi
#2}}
\providecommand{\BIBdecl}{\relax}
\BIBdecl

\bibitem{jain_provable_2013}
P.~Jain and I.~S. Dhillon, ``Provable inductive matrix completion,''
  \emph{arXiv preprint arXiv:1306.0626}, 2013.

\bibitem{rao_collaborative_2015}
N.~Rao, H.-F. Yu, P.~K. Ravikumar, and I.~S. Dhillon, ``Collaborative
  {{Filtering}} with {{Graph Information}}: {{Consistency}} and {{Scalable
  Methods}},'' in \emph{Advances in {{Neural Information Processing Systems}}
  28}, C.~Cortes, N.~D. Lawrence, D.~D. Lee, M.~Sugiyama, and R.~Garnett,
  Eds.\hskip 1em plus 0.5em minus 0.4em\relax {Curran Associates, Inc.}, 2015,
  pp. 2107--2115.

\bibitem{si_goal-directed_2016}
S.~Si, K.-Y. Chiang, C.-J. Hsieh, N.~Rao, and I.~S. Dhillon, ``Goal-{{Directed
  Inductive Matrix Completion}},'' in \emph{{{ACM SIGKDD International
  Conference}} on {{Knowledge Discovery}} and {{Data Mining}} ({{KDD}})}, Aug.
  2016.

\bibitem{donoho_when_2003}
D.~Donoho and V.~Stodden, ``When does non-negative matrix factorization give a
  correct decomposition into parts?'' in \emph{Advances in Neural Information
  Processing Systems}, 2003, p. None.

\bibitem{mei_nonnegative_2017}
\BIBentryALTinterwordspacing
J.~Mei, Y.~De~Castro, Y.~Goude, and G.~Hébrail, ``Nonnegative matrix
  factorization for time series recovery from a few temporal aggregates,'' in
  \emph{Proceedings of the 34th International Conference on Machine Learning},
  ser. Proceedings of Machine Learning Research, D.~Precup and Y.~W. Teh, Eds.,
  vol.~70.\hskip 1em plus 0.5em minus 0.4em\relax {PMLR}, pp. 2382--2390.
  [Online]. Available: \url{http://proceedings.mlr.press/v70/mei17a.html}
\BIBentrySTDinterwordspacing

\bibitem{recht_guaranteed_2010}
B.~Recht, M.~Fazel, and P.~A. Parrilo, ``Guaranteed minimum-rank solutions of
  linear matrix equations via nuclear norm minimization,'' \emph{SIAM review},
  vol.~52, no.~3, pp. 471--501, 2010.

\bibitem{cai_rop:_2015}
T.~T. Cai, A.~Zhang, and {others}, ``{{ROP}}: {{Matrix}} recovery via rank-one
  projections,'' \emph{The Annals of Statistics}, vol.~43, no.~1, pp. 102--138,
  2015.

\bibitem{zuk_low-rank_2015}
O.~Zuk and A.~Wagner, ``Low-{{Rank Matrix Recovery}} from {{Row}}-and-{{Column
  Affine Measurements}},'' in \emph{Proceedings of {{The}} 32nd {{International
  Conference}} on {{Machine Learning}}}, 2015, pp. 2012--2020.

\bibitem{velu_multivariate_2013}
R.~Velu and G.~C. Reinsel, \emph{\BIBforeignlanguage{en}{Multivariate
  {{Reduced}}-{{Rank Regression}}: {{Theory}} and {{Applications}}}}.\hskip 1em
  plus 0.5em minus 0.4em\relax {Springer Science \& Business Media}, Apr. 2013,
  google-Books-ID: dsfSBwAAQBAJ.

\bibitem{bunea_joint_2012}
F.~Bunea, Y.~She, and M.~H. Wegkamp, ``\BIBforeignlanguage{EN}{Joint variable
  and rank selection for parsimonious estimation of high-dimensional
  matrices},'' \emph{\BIBforeignlanguage{EN}{The Annals of Statistics}},
  vol.~40, no.~5, pp. 2359--2388, Oct. 2012.

\bibitem{chen_sparse_2012}
L.~Chen and J.~Z. Huang, ``Sparse reduced-rank regression for simultaneous
  dimension reduction and variable selection,'' \emph{Journal of the American
  Statistical Association}, vol. 107, no. 500, pp. 1533--1545, 2012.

\bibitem{kekatos_electricity_2014}
V.~Kekatos, Y.~Zhang, and G.~B. Giannakis, ``Electricity {{Market Forecasting}}
  via {{Low}}-{{Rank Multi}}-{{Kernel Learning}},'' \emph{IEEE Journal of
  Selected Topics in Signal Processing}, vol.~8, no.~6, pp. 1182--1193, Dec.
  2014.

\bibitem{foygel_nonparametric_2012}
R.~Foygel, M.~Horrell, M.~Drton, and J.~D. Lafferty, ``Nonparametric reduced
  rank regression,'' in \emph{Advances in {{Neural Information Processing
  Systems}}}, 2012, pp. 1628--1636.

\bibitem{candes_exact_2009}
E.~J. Cand{\`e}s and B.~Recht, ``\BIBforeignlanguage{en}{Exact {{Matrix
  Completion}} via {{Convex Optimization}}},''
  \emph{\BIBforeignlanguage{en}{Foundations of Computational Mathematics}},
  vol.~9, no.~6, pp. 717--772, 2009.

\bibitem{agarwal_regression-based_2009}
D.~Agarwal and B.-C. Chen, ``Regression-based latent factor models,'' in
  \emph{Proceedings of the 15th {{ACM SIGKDD}} International Conference on
  {{Knowledge}} Discovery and Data Mining}.\hskip 1em plus 0.5em minus
  0.4em\relax {ACM}, 2009, pp. 19--28.

\bibitem{xu_speedup_2013}
M.~Xu, R.~Jin, and Z.-H. Zhou, ``Speedup {{Matrix Completion}} with {{Side
  Information}}: {{Application}} to {{Multi}}-{{Label Learning}},'' in
  \emph{Advances in {{Neural Information Processing Systems}} 26}, C.~J.~C.
  Burges, L.~Bottou, M.~Welling, Z.~Ghahramani, and K.~Q. Weinberger,
  Eds.\hskip 1em plus 0.5em minus 0.4em\relax {Curran Associates, Inc.}, 2013,
  pp. 2301--2309.

\bibitem{chiang_matrix_2015}
K.-Y. Chiang, C.-J. Hsieh, and I.~S. Dhillon, ``Matrix {{Completion}} with
  {{Noisy Side Information}},'' in \emph{Advances in {{Neural Information
  Processing Systems}} 28}, C.~Cortes, N.~D. Lawrence, D.~D. Lee, M.~Sugiyama,
  and R.~Garnett, Eds.\hskip 1em plus 0.5em minus 0.4em\relax {Curran
  Associates, Inc.}, 2015, pp. 3447--3455.

\bibitem{abernethy_new_2009}
J.~Abernethy, F.~Bach, T.~Evgeniou, and J.-P. Vert, ``A new approach to
  collaborative filtering: {{Operator}} estimation with spectral
  regularization,'' \emph{The Journal of Machine Learning Research}, vol.~10,
  pp. 803--826, 2009.

\bibitem{candes_tight_2011}
E.~J. Cand{\`e}s and Y.~Plan, ``Tight {{Oracle Inequalities}} for
  {{Low}}-{{Rank Matrix Recovery From}} a {{Minimal Number}} of {{Noisy Random
  Measurements}},'' \emph{IEEE Transactions on Information Theory}, vol.~57,
  no.~4, pp. 2342--2359, 2011.

\bibitem{roughan_spatio-temporal_2012}
M.~Roughan, Y.~Zhang, W.~Willinger, and L.~Qiu, ``Spatio-{{Temporal Compressive
  Sensing}} and {{Internet Traffic Matrices}} ({{Extended Version}}),''
  \emph{IEEE/ACM Transactions on Networking}, vol.~20, no.~3, pp. 662--676,
  Jun. 2012.

\bibitem{rohde_estimation_2011}
A.~Rohde and A.~B. Tsybakov, ``\BIBforeignlanguage{en}{Estimation of
  high-dimensional low-rank matrices},'' \emph{\BIBforeignlanguage{en}{The
  Annals of Statistics}}, vol.~39, no.~2, pp. 887--930, 2011.

\bibitem{bhojanapalli_global_2016}
S.~Bhojanapalli, B.~Neyshabur, and N.~Srebro, ``Global {{Optimality}} of
  {{Local Search}} for {{Low Rank Matrix Recovery}},'' \emph{arXiv:1605.07221
  [cs, math, stat]}, May 2016.

\bibitem{pnevmatikakis_sparse_2013}
E.~A. Pnevmatikakis and L.~Paninski, ``Sparse nonnegative deconvolution for
  compressive calcium imaging: Algorithms and phase transitions,'' in
  \emph{Advances in {{Neural Information Processing Systems}}}, 2013, pp.
  1250--1258.

\bibitem{chen_reduced_2013}
K.~Chen, H.~Dong, and K.-S. Chan, ``Reduced rank regression via adaptive
  nuclear norm penalization,'' \emph{Biometrika}, vol. 100, no.~4, pp.
  901--920, Dec. 2013.

\bibitem{laurberg_theorems_2008}
H.~Laurberg, M.~G. Christensen, M.~D. Plumbley, L.~K. Hansen, and S.~H. Jensen,
  ``\BIBforeignlanguage{en}{Theorems on {{Positive Data}}: {{On}} the
  {{Uniqueness}} of {{NMF}}},'' \emph{\BIBforeignlanguage{en}{Computational
  Intelligence and Neuroscience}}, vol. 2008, pp. 1--9, 2008.

\bibitem{cichocki_hierarchical_2007}
A.~Cichocki, R.~Zdunek, and S.-i. Amari, ``Hierarchical {{ALS}} algorithms for
  nonnegative matrix and {{3D}} tensor factorization,'' in \emph{Independent
  {{Component Analysis}} and {{Signal Separation}}}.\hskip 1em plus 0.5em minus
  0.4em\relax {Springer}, 2007, pp. 169--176.

\bibitem{kim_algorithms_2014}
J.~Kim, Y.~He, and H.~Park, ``Algorithms for nonnegative matrix and tensor
  factorizations: A unified view based on block coordinate descent framework,''
  \emph{Journal of Global Optimization}, vol.~58, no.~2, pp. 285--319, 2014.

\bibitem{grippo_convergence_2000}
L.~Grippo and M.~Sciandrone, ``On the convergence of the block nonlinear
  {{Gauss}}\textendash{}{{Seidel}} method under convex constraints,''
  \emph{Operations Research Letters}, vol.~26, no.~3, pp. 127--136, 2000.

\bibitem{bertsekas_parallel_1989}
D.~P. Bertsekas and J.~N. Tsitsiklis, \emph{Parallel and Distributed
  Computation: Numerical Methods}.\hskip 1em plus 0.5em minus 0.4em\relax
  {Prentice hall Englewood Cliffs, NJ}, 1989, vol.~23.

\bibitem{wood_generalized_2006}
S.~Wood, \emph{Generalized Additive Models: An Introduction with {{R}}}.\hskip
  1em plus 0.5em minus 0.4em\relax {CRC press}, 2006.

\bibitem{kuhn_building_2008}
M.~Kuhn, ``Building predictive models in {{R}} using the caret package,''
  \emph{Journal of Statistical Software}, vol.~28, no.~5, pp. 1--26, 2008.

\bibitem{UCI_ML}
A.~Trindade, ``{{UCI Maching Learning Repository}} -
  {{ElectricityLoadDiagrams20112014 Data Set}},'' 2016.

\bibitem{_movielens_2015}
``{{MovieLens 100K Dataset}},''
  \url{https://grouplens.org/datasets/movielens/100k/},
  2015-09-23T15:02:16+00:00.

\bibitem{addy_reduced-rank_????}
C.~Addy, ``Reduced-{{Rank Regression}} [{{R}} package rrr version 1.0.0].''

\bibitem{yu_high-dimensional_2015}
H.-F. Yu, N.~Rao, and I.~S. Dhillon, ``High-dimensional {{Time Series
  Prediction}} with {{Missing Values}},'' \emph{arXiv preprint
  arXiv:1509.08333}, 2015.

\bibitem{karatzoglou_kernlab-s4_2004}
A.~Karatzoglou, A.~Smola, K.~Hornik, and A.~Zeileis, ``Kernlab-an {{S4}}
  package for kernel methods in {{R}},'' 2004.

\bibitem{mazumder_spectral_2010}
R.~Mazumder, T.~Hastie, and R.~Tibshirani, ``Spectral regularization algorithms
  for learning large incomplete matrices,'' \emph{Journal of machine learning
  research}, vol.~11, no. Aug, pp. 2287--2322, 2010.

\end{thebibliography}

\end{document}